\documentclass[pmlr,twocolumn,10pt]{jmlr}

\usepackage{booktabs}
\usepackage{siunitx}
\usepackage{dsfont}

\DeclareMathOperator*{\argmin}{arg\,min}
\DeclareMathOperator*{\argmax}{arg\,max}

\jmlrvolume{LEAVE UNSET}
\jmlryear{2023}
\jmlrsubmitted{LEAVE UNSET}
\jmlrpublished{LEAVE UNSET}
\jmlrworkshop{Conference on Health, Inference, and Learning (CHIL) 2023}

\title{Large-Scale Study of Temporal Shift in Health Insurance Claims}

\author{
\Name{Christina X Ji} \Email{cji@mit.edu}\\
\addr{MIT CSAIL and IMES, Cambridge, MA}\\
\Name{Ahmed M Alaa} \Email{amalaa@berkeley.edu}\\
\addr{UC Berkeley and UCSF, Berkeley, CA}\\
\Name{David Sontag} \Email{dsontag@csail.mit.edu}\\
\addr{MIT CSAIL and IMES, Cambridge, MA}
}

\begin{document}

\maketitle

\begin{abstract}

Most machine learning models for predicting clinical outcomes are developed using historical data. Yet, even if these models are deployed in the near future, dataset shift over time may result in less than ideal performance. To capture this phenomenon, we consider a task---that is, an outcome to be predicted at a particular time point---to be non-stationary if a historical model is no longer optimal for predicting that outcome. We build an algorithm to test for temporal shift either at the population level or within a discovered sub-population. Then, we construct a meta-algorithm to perform a retrospective scan for temporal shift on a large collection of tasks. Our algorithms enable us to perform the first comprehensive evaluation of temporal shift in healthcare to our knowledge. We create 1,010 tasks by evaluating 242 healthcare outcomes for temporal shift from 2015 to 2020 on a health insurance claims dataset. 9.7\% of the tasks show temporal shifts at the population level, and 93.0\% have some sub-population affected by shifts. We dive into case studies to understand the clinical implications. Our analysis highlights the widespread prevalence of temporal shifts in healthcare.

\end{abstract}

\paragraph*{Data and Code Availability}
Our experiments use a large private health insurance claims dataset. The dataset is in the Observational Medical Outcomes Partnership (OMOP) Common Data Model (CDM) v6 format \citep{hripcsak2015observational}. We provide our code at \url{https://github.com/clinicalml/large-scale-temporal-shift-study}. Our code implements our algorithms and provides an example of how to set up a large-scale scan following the guidelines we outline in \appendixref{app:guidelines}. Our repository can be used to reproduce our experiments on other datasets in the standard OMOP format.

\paragraph*{Institutional Review Board (IRB)}
This research was ruled exempt by MIT's IRB (protocol E-4025).

\section{Introduction}
\label{sec:intro}

Ensuring models are safe before deployment is a critical aspect of machine learning for healthcare \citep{wiens2019no,cohen2021problems,us2019artificial,sendak2020real}. However, validating a model solely before deployment is insufficient due to dataset shift. For example, a sepsis alert model deployed at University of Michigan Hospital started giving spurious alerts in April 2020 due to changes in patient demographics during the COVID-19 pandemic and needed to be deactivated \citep{finlayson2021clinician}. Shifts in treatment practices, technology, or data availability may also lead to poor performance of historical models. While recent works have examined the effects of temporal shifts in healthcare  \citep{guo2021systematic,guo2022evaluation,zhou2022model}, understanding the full scope of clinical dataset shift over time still requires further exploration.

\begin{figure*}[htbp]
    \floatconts
    {fig:single_test_overview}
    {\vspace{-15px}\caption{To test for temporal shift under our definition, models are first learned independently for the two time periods. $X$ are features observed in a previous time window. $Y$ is an outcome in a future window. Then, both models are evaluated with metric $\phi$ on data from the second time period. If the second model significantly outperforms the first, we define the task as non-stationary.\\\\We also include an automated sub-population discovery process in our test. First, sub-population labels $Z$ are assigned for each sample by computing the difference between the losses of the two outcome models. Then, a sub-population model $\hat{h}_t$ is fit using these labels from the training and validation data. Finally, predictions from the sub-population model $\hat{h}_t$ are used to define the sub-population on which the metric is evaluated.}}
    {\includegraphics[width=0.96\linewidth]{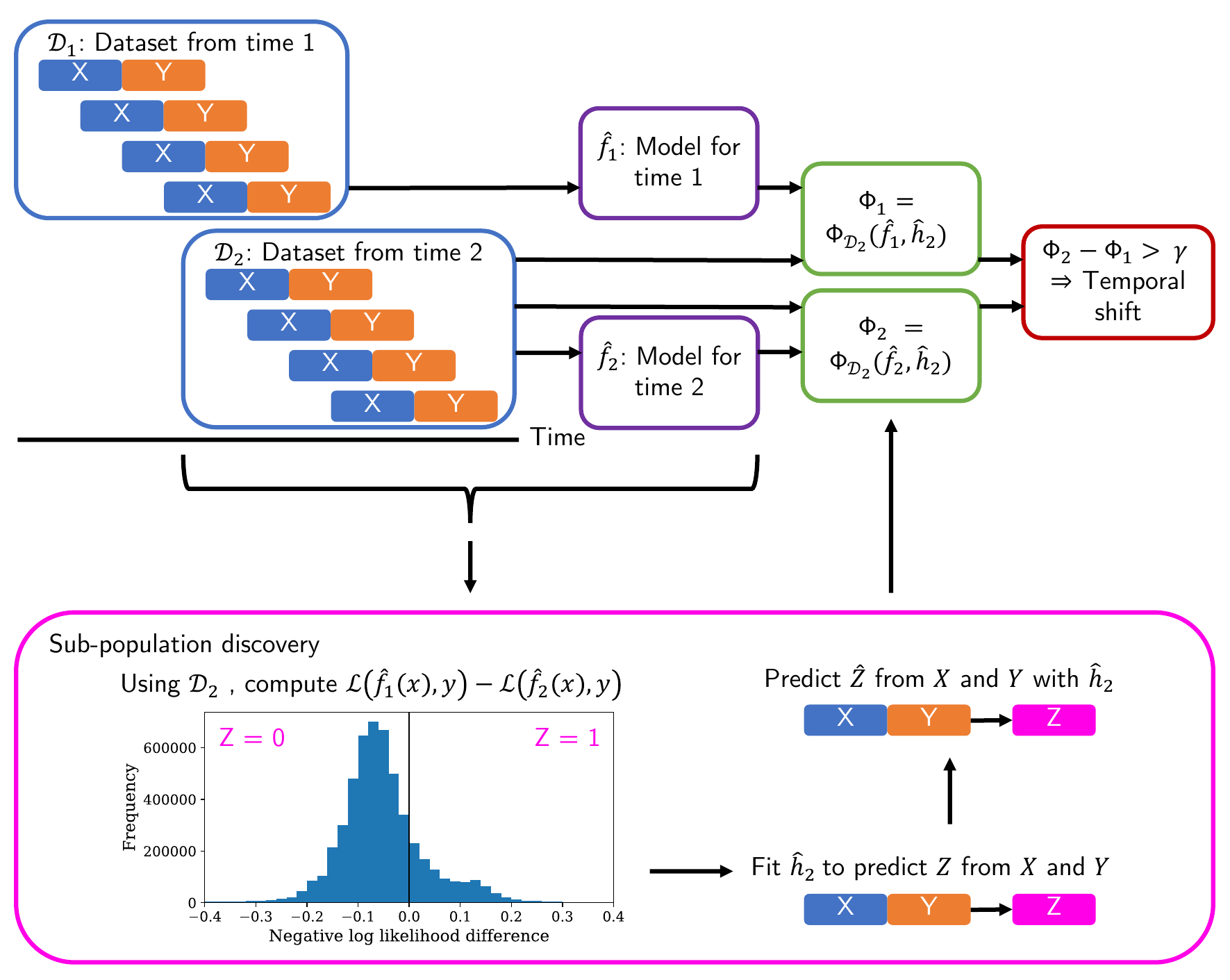}}
    \vspace{-20px}
\end{figure*}

\begin{figure*}[htb]
    \floatconts
    {fig:alg_overview}
    {\vspace{-15px}\caption{The inner algorithm tests for temporal shift in an outcome at time $t$. Two models are fit as in \figureref{fig:single_test_overview}. Then, the algorithm learns a sub-population more likely to be non-stationary by building a classifier to predict samples where the previous model has higher loss. When evaluating the entire population or a pre-defined sub-population, this step is omitted. Next, the inner algorithm tests for significantly different metric values as shown in \figureref{fig:single_test_overview}. The outer algorithm scans for temporal shift across multiple outcomes, at multiple time points, and for multiple sub-population choices. It runs the inner algorithm independently for each task. Then, it selects a set of non-stationary tasks while controlling for false discovery rate and applies another filter for clinical significance.}}
    {\includegraphics[width=0.96\linewidth]{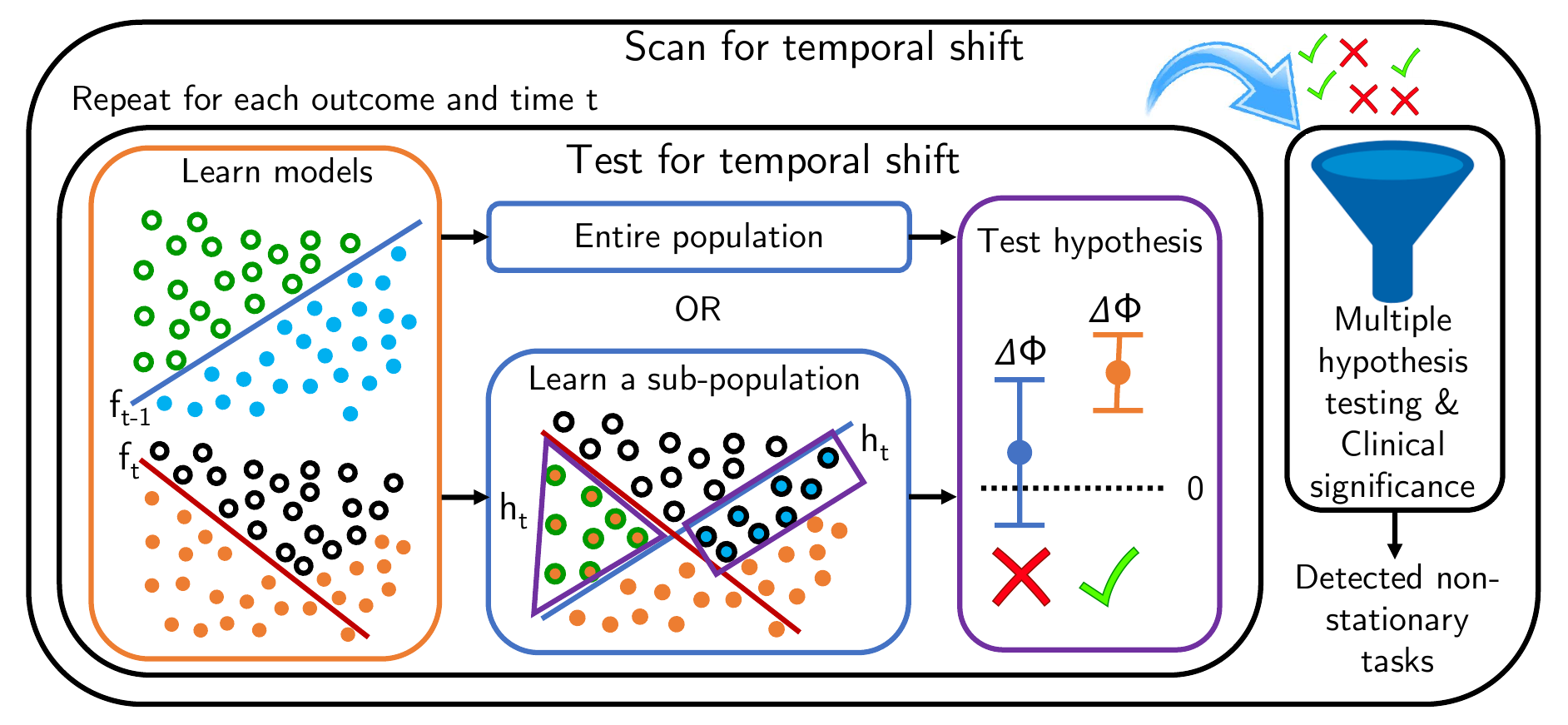}}
    \vspace{-20px}
\end{figure*}

In this work, we create an approach to perform comprehensive retrospective evaluations of temporal shift in healthcare. Motivated by the impact during model deployment, we ask two questions: 1) When is a model no longer optimal due to dataset shift? 2) Which patient groups are more adversely affected by continuing to use the outdated model?

To answer the first question, in \sectionref{sec:shift_alg}, we construct a definition for temporal shift. \figureref{fig:single_test_overview} highlights a critical aspect of our definition: To assess a task---that is, an outcome to be predicted at a particular time point, we compare the performance of models trained at two time points evaluated on data from the later time point. We build an algorithm to test this definition in \sectionref{sec:scan_alg}. Our algorithm also answers the second question by learning a sub-population and assessing shift only within that sub-population. Then, we construct a meta-algorithm that uses our algorithm to scan for temporal shift among a large collection of tasks. \figureref{fig:alg_overview} gives an overview of our algorithms. 

Our meta-algorithm enables us to perform a large-scale retrospective scan of temporal shift among 242 healthcare outcomes from 2015 to 2020. We describe our experiment on a health insurance claims dataset in \sectionref{sec:scan_setup}. To our knowledge, this work is the first comprehensive evaluation of temporal shifts with hundreds of outcomes using such a claims dataset. The outcomes are selected based on the most frequent conditions, procedures, and abnormal lab measurements in our cohort. In \sectionref{sec:quant_results}, we quantify how our meta-algorithm detects temporal shift in 9.7\% of the tasks at the population level. We also find sub-populations affected by shifts in 93.0\% of the tasks. To understand these shifts, we dive into examples of label, domain, and conditional shift in \sectionref{sec:case_studies}. The impact of these shifts must be considered when building and deploying models in healthcare. With our code and guidelines for large studies on health insurance claims data in \appendixref{app:guidelines}, our analysis can be repeated on other datasets to draw more generalized conclusions about temporal shifts in healthcare.

\section{Defining Temporal Shift}
\label{sec:shift_alg}

We are motivated by the practical goal of detecting when models need to be updated to address shifts. Let $\mathcal{E}$ be the set of all outcomes we are predicting. We first focus on testing whether a single outcome $e \in \mathcal{E}$ is affected by shift between times $t-1$ and $t$ within a particular patient population. Each choice of outcome $e$, time $t$, and patient population constitutes a task. In this section, we present a definition for when a task is affected by temporal shift. 

\subsection{Outcome Model, Sub-population, and Metric}
\label{sec:comp_models}

To construct our definition for temporal shift, we introduce three pieces: a model predicting the outcome, a classifier defining the sub-population, and a metric.

\textbf{Outcome model}: Let $\mathcal{Y} \in \left\{0, 1\right\}$ or $\mathbb{R}$ be the outcome space and $\mathcal{X} \in \mathbb{R}^d$ be the feature space. Let $\mathcal{L}: \mathcal{R} \times \mathcal{Y} \rightarrow \mathbb{R}_{\ge 0}$ be a loss function that evaluates a prediction against the true label. The model $\hat{f}_t$ minimizes the empirical risk over dataset $\mathcal{D}_t$:
\begin{equation}
    \hat{f}_t = \argmin_{f \in \mathcal{F}} \frac{1}{\lvert \mathcal{D}_t \rvert} \sum_{\left(x, y\right) \in \mathcal{D}_t} \mathcal{L}\left(f\left(x\right), y\right)
\end{equation}
$\hat{f}_{t-1}$ and $\hat{f}_t$ are learned independently.

\textbf{Sub-population}: The population to evaluate can consist of all patients, a pre-defined subgroup based on clinical knowledge, or a learned sub-population. Let the population membership at time $t$ be defined by a binary indicator in $\mathcal{Z} \in \left\{0, 1\right\}$. A classifier $h_t: \mathcal{X} \times \mathcal{Y} \rightarrow \mathcal{Z}$ maps each sample to its membership indicator. The outcome may be used in the sub-population definition. For instance, if cancer screenings start at 45 instead of 50 at time $t$, then $\hat{f}_{t-1}$ is more likely to miss cancers detected among 45 to 50 year olds because they were not previously screened. Thus, a sub-population that is likely to be affected by temporal shift would be patients who are 45 to 50 years old and who are diagnosed with cancer. 

\textbf{Metric:} Let $\phi_{\mathcal{D}}(\hat{f}, h)$ denote the value of metric $\phi$ when evaluating $\hat{f}$ within a sub-population specified by $h$ on dataset $\mathcal{D}$. The metric $\phi$ is different from the loss $\mathcal{L}$ for training outcome models. To capture how temporal shift affects downstream applications, $\phi$ can be chosen to reflect how models are used in clinical practice. Unlike $\mathcal{L}$, $\phi$ may not be valid for performing gradient descent. Furthermore, while loss is minimized, higher values of $\phi$ are better.

\subsection{Temporal Shift Definition}
\label{sec:definition}

We construct a definition for temporal shift that captures when a model needs to be updated. 
\begin{definition}[Temporal Shift]
    \label{def:temporal_shift}
    For an outcome at time $t$ in a sub-population defined by $h_t$, temporal shift occurs if the fitted models $\hat{f}_{t-1}$ and $\hat{f}_t$ satisfy the following property:
    \begin{equation}
        \label{eq:def_shift}
        \phi_{\mathcal{D}_t}\left(\hat{f}_t, h_t\right) - \phi_{\mathcal{D}_t}\left(\hat{f}_{t-1}, h_t\right) > 0
    \end{equation}
\end{definition}
The first term measures how an optimal model trained on current data performs in-distribution. The second term evaluates how an optimal historical model performs at the current time point. A task affected by temporal shift is \textit{non-stationary}.

We compare model performance instead of the distributions $\mathbb{P}_t\left(X, Y\right)$ and $\mathbb{P}_{t-1}\left(X, Y\right)$ for two reasons. First, there is insufficient coverage over the high-dimensional $\mathcal{X}$ space to estimate $\mathbb{P}\left(Y \vert X\right)$. The conditional distribution is much harder to characterize than the label distribution that is tested in \citet{lipton2018detecting}. Fitting $\hat{f}_t$ and $\hat{f}_{t-1}$ approximates $\mathbb{P}\left(Y \vert X\right)$ where the model may be used. Second, our definition uses a metric $\phi$ that can be chosen to reflect clinical impact. Distance functions between distributions tend to be more sensitive to particular aspects, such as the median or spread. Unlike $\phi$, these aspects may not reflect changes in clinical decision-making due to using an outdated model.

Evaluating $\hat{f}_{t-1}$ on $\mathcal{D}_t$ in the second term of \equationref{eq:def_shift} is intuitive since it captures how a historical model performs at the current time point. Notably, the first term evaluates the current model $\hat{f}_t$ on current data $\mathcal{D}_t$ rather than the previous model $\hat{f}_{t-1}$ on previous data $\mathcal{D}_{t-1}$. That is, we are not seeking the property: $\phi_{\mathcal{D}_{t-1}}\left(\hat{f}_{t-1}, h_t\right) - \phi_{\mathcal{D}_t}\left(\hat{f}_{t-1}, h_t\right) > 0$. We create a baseline that evaluates this property and defer to \sectionref{sec:baseline} for further comparison.

\section{Algorithms to Test and Scan for Temporal Shift}
\label{sec:scan_alg}

We present two algorithms: The first tests whether \definitionref{def:temporal_shift} holds in a single task. The second scans for temporal shift among a collection $\mathcal{E}$ of outcomes across multiple time points. Before constructing our algorithms, we present two integral pieces: assessing \definitionref{def:temporal_shift} via hypothesis testing and automated sub-population discovery.

\subsection{Hypothesis Testing}
\label{sec:hyp_test}

We perform hypothesis testing to assess whether \definitionref{def:temporal_shift} for temporal shift holds. Prior to testing, we check that sample size is sufficient and models are well-fit. To avoid conducting unnecessary hypothesis tests for tasks unlikely to be affected by temporal shift, we only perform the test if $\hat{f}_t$ performs significantly better than $\hat{f}_{t-1}$ on validation data. Details on these criteria are in \appendixref{app:criteria}.

Our algorithm examines if there is significant distribution shift when predicting an outcome in a particular year. The null hypothesis states models learned on data from the current year versus the previous year perform equally well when evaluated on data from the current year. We test this hypothesis against a one-sided alternative:

\begin{align}
    H_0: \phi_{\mathcal{D}_t}\left(\hat{f}_t, h_t\right) - \phi_{\mathcal{D}_t}\left(\hat{f}_{t-1}, h_t\right) = 0\\
    H_1: \phi_{\mathcal{D}_t}\left(\hat{f}_t, h_t\right) - \phi_{\mathcal{D}_t}\left(\hat{f}_{t-1}, h_t\right) > 0
\end{align}
A more formal statement of these hypotheses is given in \appendixref{app:permutation_test}. Note that if $\phi$ is defined as the log likelihood, our test can be viewed as a likelihood ratio test. As discussed in Section~\ref{sec:definition}, our test is more general. To evaluate any $\phi$, we perform a variation of the permutation test defined in \citet{bandos2005permutation}. Our test is specified in \algorithmref{alg:permut_2models_1dataset}.

\subsection{Discovering Sub-populations with Shifts}

\label{sec:subpop_step}

Our goal is to identify a sub-population where temporal shift is likely. That is, we would like to learn an $h_t$ where \definitionref{def:temporal_shift} is satisfied. Our objective is to maximize the difference between the losses of the previous and current model within the sub-population:
\begin{multline}
    h_t^* = \argmax_{h} \mathbb{E}_{\mathcal{D}_t}\Bigl[\mathds{1}\left\{h\left(x, y\right)\right\} \times \\
    \left(\mathcal{L}\left(\hat{f}_{t-1}\left(x\right), y\right) - \mathcal{L}\left(\hat{f}_t\left(x\right), y\right)\right)\Bigr]
\end{multline}
An optimal $h_t^*$ will select all samples $\left(x, y\right)$ where $\hat{f}_{t-1}$ has higher loss. We explain why this approach is better than alternative formulations in \appendixref{app:subpop_problem}. The sub-population indicators $z \in \mathcal{Z}$ are defined as
\begin{equation}
    \label{eq:subpop_label}
    z := \mathds{1}\left\{\mathcal{L}\left(\hat{f}_{t-1}\left(x\right), y\right) - \mathcal{L}\left(\hat{f}_t\left(x\right), y\right) > 0\right\}
\end{equation}
As shown in \figureref{fig:single_test_overview}, we fit a model $\hat{h}_t: \mathcal{X} \times \mathcal{Y} \rightarrow \mathcal{Z}$ with these labels. Defining sub-populations with $\hat{h}_t$ rather than labels $z$ for the test samples ensures the conclusions are more generalizable. We can also restrict the hypothesis class $\mathcal{H}$ for $\hat{h}_t$ to create more interpretable or clinically meaningful sub-populations.

\subsection{Algorithm to Test for Temporal Shift}

\begin{algorithm2e}[h]
    \caption{Test for temporal shift}
    \label{alg:test_shift}
    \KwIn{Datasets $\mathcal{D}_{t-1} = \left(X^{t-1}_i, Y^{t-1}_i\right)_{i=1}^{m}$ and $\mathcal{D}_t = \left(X^t_i, Y^t_i\right)_{i=1}^{n}$ in data splits, sub-population model $h_t$ for a pre-defined sub-population (default: None; specify $h_t\left(x, y\right) = 1 \forall x \in \mathcal{X}, y \in \mathcal{Y}$ for entire population)}
    \KwOut{p-value, metric difference, sub-population model if hypothesis test is performed}
    \textsf{\upshape Fit} $\hat{f}_{t-1}\left(x\right)$ \textsf{\upshape to minimize} $\frac{1}{m} \sum_{i=1}^m\mathcal{L}\left(f\left(x^{t-1}_i\right), y^{t-1}_i\right)$ \textsf{\upshape using training and validation samples}\; 
    \textsf{\upshape Fit} $\hat{f}_t\left(x\right)$ \textsf{\upshape to minimize} $\frac{1}{n} \sum_{i=1}^n\mathcal{L}\left(f\left(x^t_i\right), y^t_i\right)$ \textsf{\upshape using training and validation samples}\; 
    \If{$h_t$ \textsf{\upshape is not pre-specified}}{
        \For{$i \leftarrow 1$ \KwTo $n$}{
            $z_i^t \gets \mathds{1} \left\{ \mathcal{L}\left(\hat{f}_{t-1}\left(x_i^t\right), y_i^t\right) \right.$ \\ $\quad \quad \quad \quad \left.- \mathcal{L}\left(\hat{f}_t\left(x_i^t\right), y_i^t\right) > 0\right\}$\;
        }
        \textsf{\upshape Shuffle training and validation splits}\;
        \textsf{\upshape Fit} $\hat{h}_t\left(x, y\right)$ \textsf{\upshape to minimize} $\frac{1}{n} \sum_{i=1}^n \mathcal{L}\left(h\left(x_i^t, y_i^t\right), z_i^t\right)$ \textsf{\upshape using new splits}\;
    }
    \If{\textsf{\upshape sample size check in \algorithmref{alg:sample_size_check} returns false with validation samples}, $\hat{h}_t$\textsf{\upshape or }$h_t$\textsf{\upshape as inputs}}{
        \KwOut{\textsf{\upshape None}}
    }
    \If{\textsf{\upshape model fit check in \algorithmref{alg:model_fit_check} returns false with validation samples}, $\hat{f}_{t-1}, \hat{f}_t, \hat{h}_t$\textsf{\upshape or }$h_t$\textsf{\upshape as inputs}}{
        \KwOut{\textsf{\upshape None}}
    }
    \If{\textsf{\upshape performance comparison in \algorithmref{alg:perform_comp} returns false with validation samples}, $\hat{f}_{t-1}, \hat{f}_t, \hat{h}_t$\textsf{\upshape or }$h_t$\textsf{\upshape as inputs}}{
        \KwOut{\textsf{\upshape None}}
    }
    $p \leftarrow $ \textsf{\upshape one-sided p-value for}\\ $H_0: \phi_{\mathcal{D}_t}\left(\hat{f}_t, h_t\right) - \phi_{\mathcal{D}_t}\left(\hat{f}_{t-1}, h_t\right) = 0$
    \textsf{\upshape from \algorithmref{alg:permut_2models_1dataset} with test samples}, $\hat{f}_{t-1}, \hat{f}_t, \hat{h}_t$\textsf{\upshape or }$h_t$\textsf{\upshape as inputs}\;
    $a \leftarrow \phi_{\mathcal{D}_t}\left(\hat{f}_t, \hat{h}_t\right) - \phi_{\mathcal{D}_t}\left(\hat{f}_{t-1}, \hat{h}_t\right)$\;
    \KwOut{$p$, $a$, $\hat{h}_t$ \textsf{\upshape or }$h_t$}
\end{algorithm2e}

We integrate the ideas we have introduced to construct \algorithmref{alg:test_shift}. This algorithm takes in datasets from times $t-1$ and $t$ and tests for temporal shift between those time points. A pre-defined sub-population may also be specified. To test at the population level, we can input $h_t\left(x, y\right) = 1 \forall x \in \mathcal{X}, y \in \mathcal{Y}$.

The first stage of the algorithm fits the component models described in \sectionref{sec:comp_models}. First, separate outcome models for times $t-1$ and $t$ are fit. Then, a sub-population is discovered using the objective in \sectionref{sec:subpop_step}. Labels for the sub-population model are computed using \equationref{eq:subpop_label}. The training and validation sets are re-split in case the sub-population label distributions differ between the two splits. Lastly, the sub-population model is fit using the new labels.

The second stage of the algorithm is hypothesis testing. First, the checks described in \sectionref{sec:hyp_test} are performed to determine whether the task should be tested. Up to this point, the test set is held out. In the final step, the algorithm performs hypothesis testing on this held-out test set. The p-value, discovered sub-population, and observed metric difference are returned. The final output can be used to assess clinical significance.

\subsection{Meta-Algorithm to Scan for Temporal Shift}

Now that we are equipped with \algorithmref{alg:test_shift} to test for temporal shift in a single task, we construct a meta-algorithm to assess temporal shift across a large collection of tasks. \algorithmref{alg:scan_shift} takes in datasets for each outcome at each time point, a desired false discovery rate for the statistical significance filter, and a threshold for clinically significant metric differences. It outputs a list of tasks where temporal shifts have statistically and clinically significant effects.

The first phase treats each task as an independent problem. For each outcome at each time point, we run \algorithmref{alg:test_shift} once at the population level and once within discovered sub-populations. Temporal shift may have detrimental effects for a sub-population while not appearing significant at the population level. Specifying sub-populations based on clinical knowledge for each outcome may be very labor-intensive. We also may not have prior knowledge on which sub-populations are affected by temporal shift. Thus, automated sub-population discovery plays an essential role in large-scale scans.

The second phase gathers results from the individual tests and accounts for both statistical and clinical significance. To ensure we are not overstating the prevalence of temporal shift, we impose a 5\% false discovery rate. Because the test split is held out prior to hypothesis testing, the checks in \algorithmref{alg:test_shift} noticeably reduce the number of hypothesis tests that need to be accounted for when applying the Benjamini-Hochberg correction \citep{benjamini1995controlling}. We also evaluate the clinical impact of temporal shift via the metric difference. When this difference is above a threshold for a task, the temporal shift is considered clinically significant.

\begin{algorithm2e}[t]
    \caption{Scan for temporal shift}
    \label{alg:scan_shift}
    \SetNoFillComment
    \KwIn{Datasets $\left\{\left(X_i^t, Y_i^t\right)_{i=1}^{n_t}\right\}_{t=0}^T$ at multiple time points for each outcome $e \in \mathcal{E}$, false discovery rate $\alpha$ (default: .05), clinical significance threshold $\gamma$ (default: .01)}
    \KwOut{Tasks $\left\{\left(e, t, \hat{h}_t\right)\right\}$ with statistically and clinically significant temporal shift}
    $P \leftarrow \left\{\right\}$ \tcp*{Map task to p-value}
    $A \leftarrow \left\{\right\}$ \tcp*{Map task to metric diff}
    \For{\textsf{\upshape outcome} $e \in \mathcal{E}$}{
        \For{$t \leftarrow 1$ \KwTo $ T$}{
            \For{$h_t\left(x, y\right) = 1 \forall x \in \mathcal{X}, y \in \mathcal{Y}$ \textsf{\upshape and } $h_t$ \textsf{\upshape as None}}{
                $p, a, \hat{h}_t \leftarrow$ \textsf{\upshape \algorithmref{alg:test_shift} with inputs} $\left(X_i^{t-1}, Y_i^{t-1}\right)_{i=1}^{n_{t-1}}, \left(X_i^t, Y_i^t\right)_{i=1}^{n_t}, h_t$\;
                \If{$p$ \textsf{\upshape is not None}}{
                    $P\left[\left(e, t, \hat{h}_t\right)\right] \leftarrow p$\;
                    $A\left[\left(e, t, \hat{h}_t\right)\right] \leftarrow a$\;
                }
            }
        }
    }
    $Q \leftarrow $ \textsf{\upshape Benjamini-Hochberg}$\left(P, \alpha\right)$\;
    $R \leftarrow \left\{\left(e, t, \hat{h}_t\right): A\left[\left(e, t, \hat{h}_t\right)\right] > \gamma \right\}$\; 
    \KwOut{$Q \cap R$}
\end{algorithm2e}

\section{Large-Scale Scan on Health Insurance Claims}
\label{sec:scan_setup}

We demonstrate our algorithms in a large-scale scan for temporal shifts on a health insurance claims dataset. Such datasets are powerful resources for studying temporal shift because they provide a snapshot of patient state at each visit, a longitudinal view of all visits while enrolled, and a large insured population. We found these ingredients to be essential for performing a large-scale scan: standardized data formats, a clear definition for a sequence of cohorts, selection of frequent outcomes, lab measurement cleaning, and efficient automated feature extraction. In \appendixref{app:guidelines}, we share how we handled each component in our scan to empower similar analyses.

\subsection{Experiment Set-up}
\label{sec:scan_setup_details}
We use a large de-identified health insurance claims dataset  comprised mostly of patients in Pennsylvania and New Jersey. Since our dataset spans 2014 to mid-2021, we assess temporal shifts on a yearly basis from 2015 to 2020. Our collection of outcomes $\mathcal{E}$ contains 100 initial condition diagnoses, 100 abnormal lab measurements, and 42 procedure groups. Examples of condition outcomes include cough, acute bronchitis, and type 2 diabetes. Lab outcomes include mean corpuscular hemoglobin concentration below 33 g/dL, glucose above 125 mg/dL, and low-density lipoprotein cholesterol above 129 mg/dL. The procedure outcomes are defined by groups of CPT-4, HCPCS, ICD-9, ICD-10, and SNOMED concepts. For example, the office visit and surgery outcomes are defined by 27 and 514 codes, respectively. $Y \in \left\{0, 1\right\}$ indicates whether the outcome occurs in the next 3 months.

The feature space $\mathcal{X}$ includes over 15,000 features related to demographics, prediction month, conditions, procedures, drugs, lab measurements, and specialty visits in the past 30 days. We justify this feature window choice in \appendixref{app:feature_window_choice}. In our cohort inclusion criteria, we require the patient is observed for a period of time before and after the prediction date. The lengths of these periods are given in \appendixref{app:cohort_def}. Our cohort includes over 1.6 million patients. Each patient may contribute one sample per prediction month, with an average of 34 samples per patient. For condition outcomes, we exclude patients who were previously diagnosed with the condition and require additional years of prior observation to check this criterion. Thus, we only assess temporal shift for condition outcomes from 2017 to 2020. The smaller cohorts for condition outcomes include around 0.8 million patients, with an average of 29 samples per person.

The extensive collection of outcomes $\mathcal{E}$, high-dimensional feature space $\mathcal{X}$, and large cohorts over 6 years enable our comprehensive scan of temporal shifts. We examine 1,010 tasks at both the population level and within discovered sub-populations.

\subsection{Modeling Choices}
\label{sec:modeling_choices}
To implement \algorithmref{alg:test_shift}, we need to specify a class $\mathcal{F}$ of outcome models, a corresponding definition for sub-population labels $\mathcal{Z}$, a class $\mathcal{H}$ of sub-population models, and a metric $\phi$.

We choose AUC as the metric $\phi$ because it handles class imbalance. For outcome models, we use logistic regressions for our hypothesis class $\mathcal{F}$ since they achieve higher AUC in the experiments in \appendixref{app:outcome_model}. Logistic regressions are also easier to interpret when examining how models change over time. With different modeling choices, the evaluation of \definitionref{def:temporal_shift} may give different results. Nevertheless, any well-fit model can be used to assess temporal shift.

We prove in \theoremref{thm:subpop_labels_equiv} in \appendixref{app:subpop_crossent_calibration} that the cross entropy loss for logistic regressions gives rise to a calibration-based definition for the sub-population labels $z_i$. Calibration captures more nuanced differences than misclassification since the threshold for predicting 1 may be shifted for different clinical applications. We use decision trees as our hypothesis class $\mathcal{H}$ for sub-population models since this leads to the most non-stationary regions in experiments shown in \appendixref{app:subpop_model}.

\section{Results: Quantifying Prevalence of Temporal Shift}
\label{sec:quant_results}

\figureref{fig:outcome_counts} shows that our method identifies 98 non-stationary tasks among the 1,010 that we scan. The non-stationary tasks are listed in \appendixref{app:scan_results}. Lab outcomes account for 32 of the 36 non-stationary tasks before 2020. Some of these cases, such as the estimated glomerular filtration rate (eGFR) outcomes, can be attributed to changes in how lab tests were recorded. The majority of affected tasks--62 out of 98--occur in 2020 since the COVID-19 pandemic had a widespread effect on the healthcare system. We will examine these shifts in detail in \sectionref{sec:case_studies}.

\begin{figure}[t]
    \floatconts
    {fig:outcome_counts}
    {\vspace{-20px}\caption{Percentage of tasks identified as non-stationary for each outcome type. Top: Our algorithm versus the baseline on the entire population. Bottom: Our algorithm with discovered sub-populations. Color: Outcome type. Line style: Algorithm.}}
    {\includegraphics[width=0.96\linewidth]{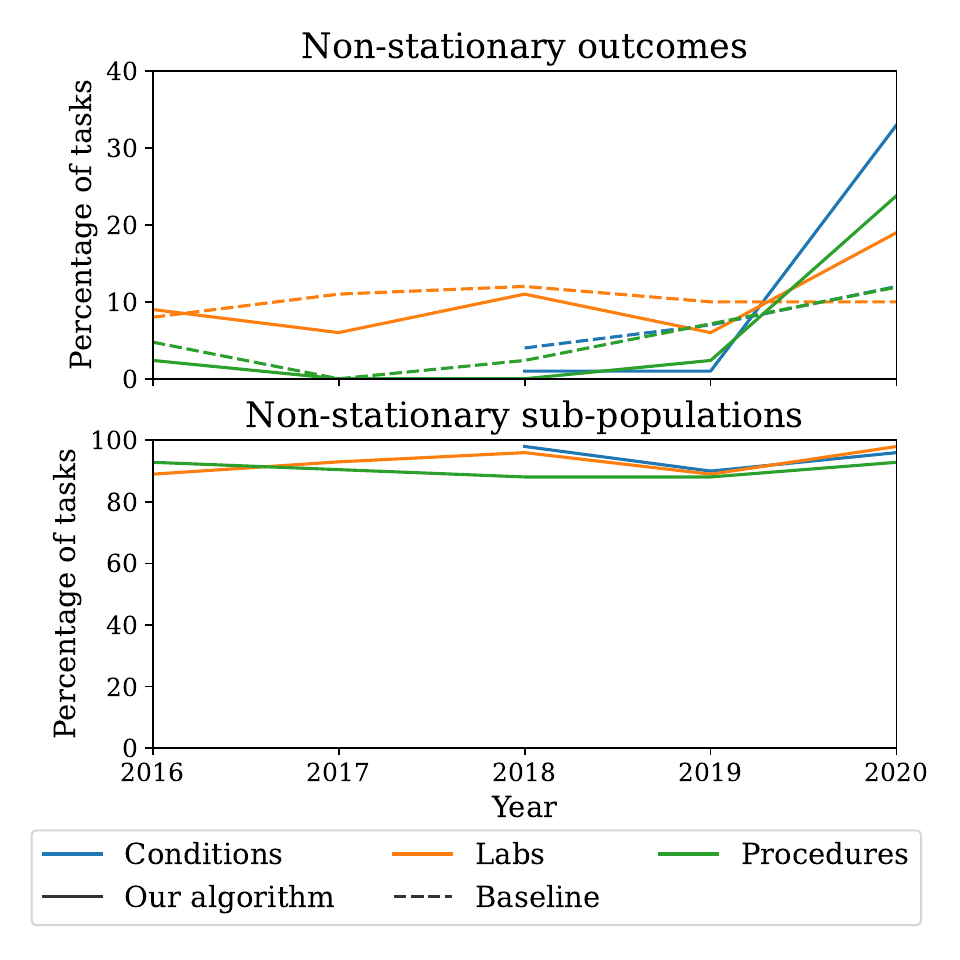}}
    \vspace{-5px}
\end{figure}

While only 9.7\% of the 1,010 tasks are affected at the population-level, 93.0\% of the tasks have at least some sub-population that is affected. Furthermore, \figureref{fig:auc_diff_dist} in \appendixref{app:scan_results} shows how the effect of temporal shift can be even larger within sub-populations. This result highlights how monitoring temporal shift at the population level may not be sufficient for determining when to update a model.

\subsection{Baseline Comparison}
\label{sec:baseline}

In \sectionref{sec:definition}, we highlighted how our definition for temporal shift compares against $\phi_{\mathcal{D}_t}(\hat{f}_t, h_t)$ rather than $\phi_{\mathcal{D}_{t-1}}(\hat{f}_{t-1}, h_t)$. Many prior benchmarks we discuss in \sectionref{sec:related} treat the latter as the in-distribution performance. We perform this comparison as a baseline. To modify \algorithmref{alg:test_shift} into the baseline, we replace the performance comparison criteria with \algorithmref{alg:baseline_comp} and the permutation test with \algorithmref{alg:permut_1model_2datasets} given in \appendixref{app:hyp_test_details}.

\begin{figure*}[t]
    \floatconts
    {fig:baseline_comparison_examples}
    {\vspace{-20px}\caption{Top: AUC of current and previous model in each year with bootstrap standard errors. Bottom: AUC difference evaluated by our algorithm (red AUC at $t$ minus blue at $t$) and baseline (red at $t-1$ minus blue at $t$) with 90\% bootstrap confidence intervals. Black dotted line: Statistical significance threshold for confidence interval. The algorithms actually assess statistical significance via permutation tests. Orange dotted line: Clinical significance threshold for point estimate. Left: Outcome is high prothrombin time lab measurement. No temporal shift in any year. Baseline erroneously detects shifts in 2018 and 2020. Right: Outcome is defined as receiving a colonoscopy exam. Temporal shift in 2020 only identified by our algorithm. No temporal shift in other years.}}
    {\includegraphics[width=0.96\linewidth]{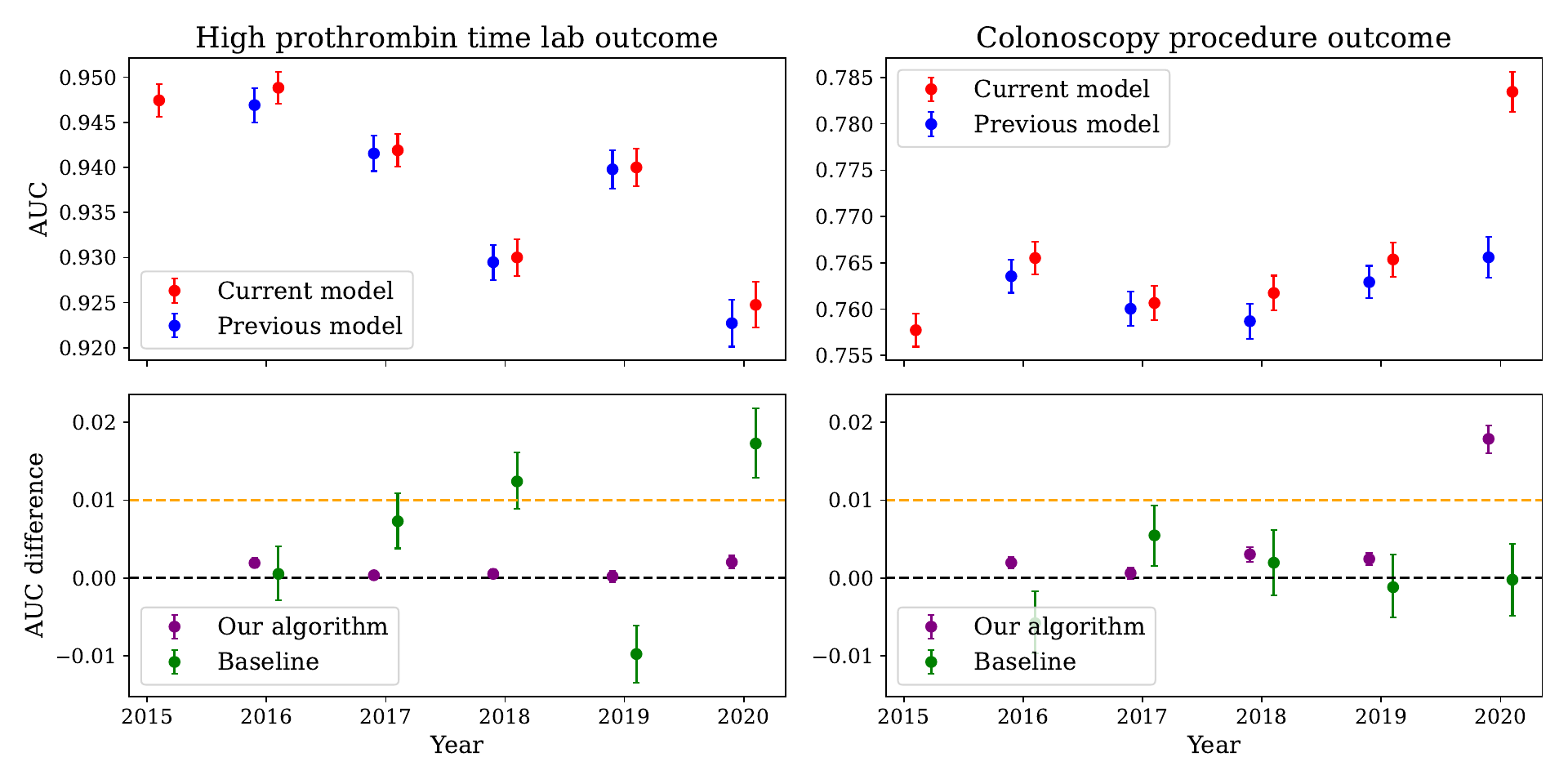}}
    \vspace{-20px}
\end{figure*}

Among 72 tasks where the baseline detects temporal shift, 35 are not identified by our algorithm since they do not satisfy our definition of temporal shift. An example is the high prothrombin time lab outcome in 2018 and 2020. This lab result may indicate slow blood clotting. \figureref{fig:baseline_comparison_examples} shows the 2017 and 2018 models perform similarly in 2018. The baseline identifies these tasks as non-stationary because the performance of the 2017 model drops from 2017 to 2018 (likewise for 2019 to 2020). However, such performance drops may be due to random fluctuations or noise unrelated to temporal shift. Comparing against $\hat{f}_t$ in our algorithm accounts for these factors.

The baseline also misses many non-stationary tasks. 61 tasks identified by our algorithm are not detected by the baseline. An example is the outcome defined as receiving a colonoscopy exam in 2020. \figureref{fig:baseline_comparison_examples} shows a large gap between the AUCs of the 2019 and 2020 models in 2020. This task is affected by label shift as outcome frequency drops from 1.2\% to 0.8\% in 2020. Our algorithm detects this shift, while the baseline does not because the 2019 model achieves a higher AUC in 2020 than 2019. We defer more discussion on the baseline to \appendixref{app:baseline}.

\section{Case Studies: Types of Temporal Shift}
\label{sec:case_studies}

To demonstrate the temporal shifts identified in our scan are driven by clinical changes, we analyze examples of three types of temporal shift shown in \figureref{fig:types_of_shift}: label shift, domain shift, and conditional shift. We propose a clinically motivated solution to address label shift. For domain and conditional shift, we discuss how addressing them without data from the new time period is challenging in healthcare settings.  

\begin{figure}[t]
    \floatconts
    {fig:types_of_shift}
    {\vspace{-15px}\caption{Three types of shifts. $D$ is the domain. Only $D$ changes between time points. $U$ contains unobserved variables.}}
    {\includegraphics[width=0.8\linewidth]{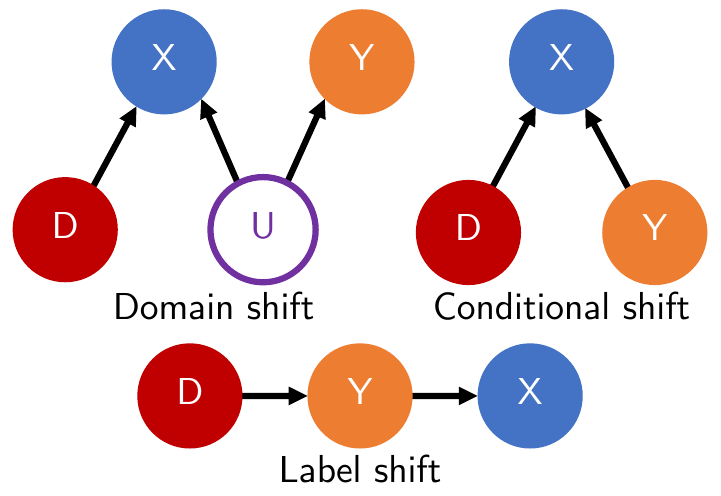}}
    \vspace{-5px}
\end{figure}

\subsection{Label Shift in eGFR Measurements}
\label{sec:egfr_label_shift}

Label shift is defined by changes in the outcome distribution $\mathbb{P}\left(Y\right)$. Label shifts can be caused by changes in the prevalence of diseases or procedures or by modifications to how features are recorded.

\textbf{Laying out the case:} The low estimated glomerular filtration rate (eGFR) outcome has a change in recording patterns in 2018. Our algorithm detects shifts in 4 of the 6 eGFR outcomes. The MDRD formula was a popular tool for computing eGFR before the CKD-EPI formula became more widely adopted \citep{levey2006using,levey2009new,american2022introduction}. With both formulas, eGFR measurements below 60 correspond to stage 3 kidney disease \citep{kidney2013kdigo}. Thus, we define eGFR below 60 as an abnormal outcome.

\begin{figure}[t]
    \floatconts
    {fig:egfr_label_shift}
    {\vspace{-30px}\caption{Proportion of patients who have a value recorded for each eGFR concept in the outcome window. For either concept, lab reports show two eGFR values: one adjusted for non-African Americans and another for African Americans, regardless of patient race. The lines depicting the first two adjustments for CKD-EPI overlap. Color: Adjustment. Line style: eGFR formula.}}
    {\includegraphics[width=0.96\linewidth]{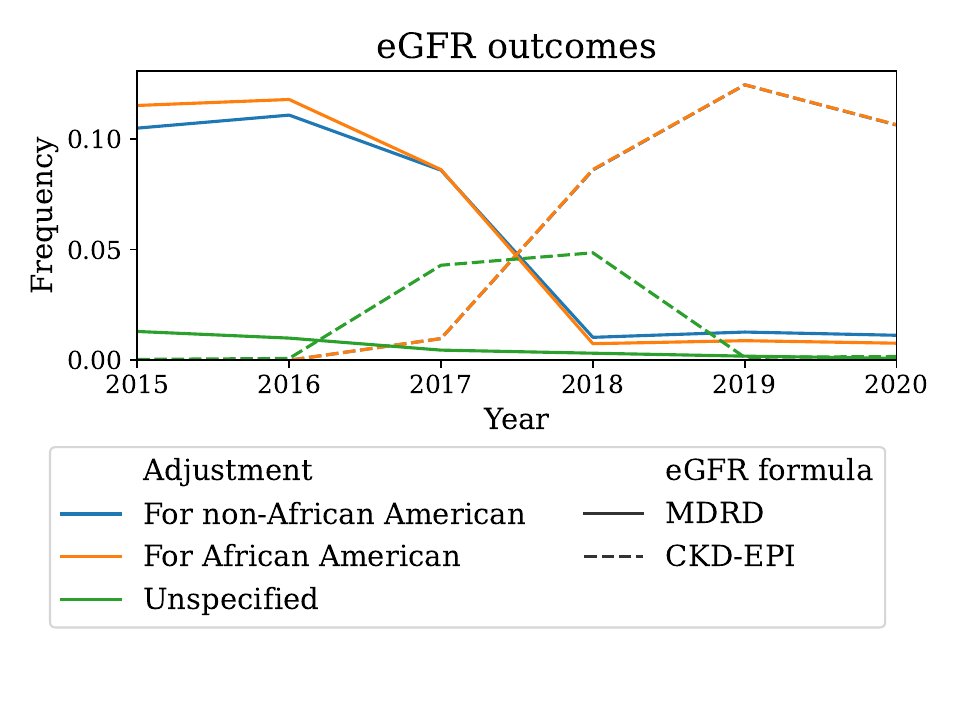}}
    \vspace{-20px}
\end{figure}

\figureref{fig:egfr_label_shift} shows this formula shift happened between 2017 and 2018 in our dataset. Models trained to predict MDRD-based eGFR outcomes in 2017 may correctly predict that a patient has low eGFR in 2018. However, because the measurement is recorded under the CKD-EPI concept, the patient does not have the low MDRD outcome, and the prediction is considered incorrect. On the other hand, the 2017 model for the CKD-EPI outcome is trained on a limited subset of low eGFR outcomes from lab centers that switched earlier and misses many low eGFR outcomes in 2018.

\textbf{Solving the case:} Label shift can be addressed by mapping $Y^{t-1}$ to $Y^t$. Instead of learning a label transformation using data from the target domain \citep{guo2020ltf},  we modify the outcome definition to handle changes in recording practices. Combining the 6 eGFR outcomes into a single outcome eliminates the underlying cause of label shift. The frequency of this combined outcome is consistently around 2.5\% to 3.1\% over the years. With this new outcome, the 2017 and 2018 models both achieve an AUC around 0.91 on the test data in 2018.

\subsection{Domain Shift in 2020}
\label{sec:cov_shift_2020}

As shown in \figureref{fig:types_of_shift}, the domain impacts how the unobserved variables $U$ are reflected in the observed covariates $X$ \citep{quinonero2008dataset}. When the domain changes, $\mathbb{P}\left(X \vert U\right)$ shifts. Because $U$ is latent, the shift we observe is in $\mathbb{P}\left(X\right)$. In domain shift, we assume $\mathbb{P}\left(Y \vert U\right)$ is unchanged.

\textbf{Laying out the case:} The domain shift mechanism is shared across outcomes at the same time point. Because 62 of the 98 non-stationary tasks are in 2020, we hypothesize domain shift during the COVID-19 pandemic may be a shared reason. Using samples from April to December of 2019 and 2020, we fit an L1-regularized logistic regression to predict which year each sample is from. Then, for each feature with a non-zero coefficient, we run a chi-squared test with the following null hypothesis: The feature frequency is the same in 2019 and 2020. After applying the Benjamini-Hochberg correction with an expected false discovery rate of 5\%, 781 hypotheses for univariate $\mathbb{P}\left(X\right)$ shifts are accepted.

\begin{figure}[t]
    \floatconts
    {fig:covariate_shift_2020}
    {\vspace{-15px}\caption{Selection of features with univariate $\mathbb{P}\left(X\right)$ shifts from 2019 to 2020}}
    {\includegraphics[width=0.96\linewidth]{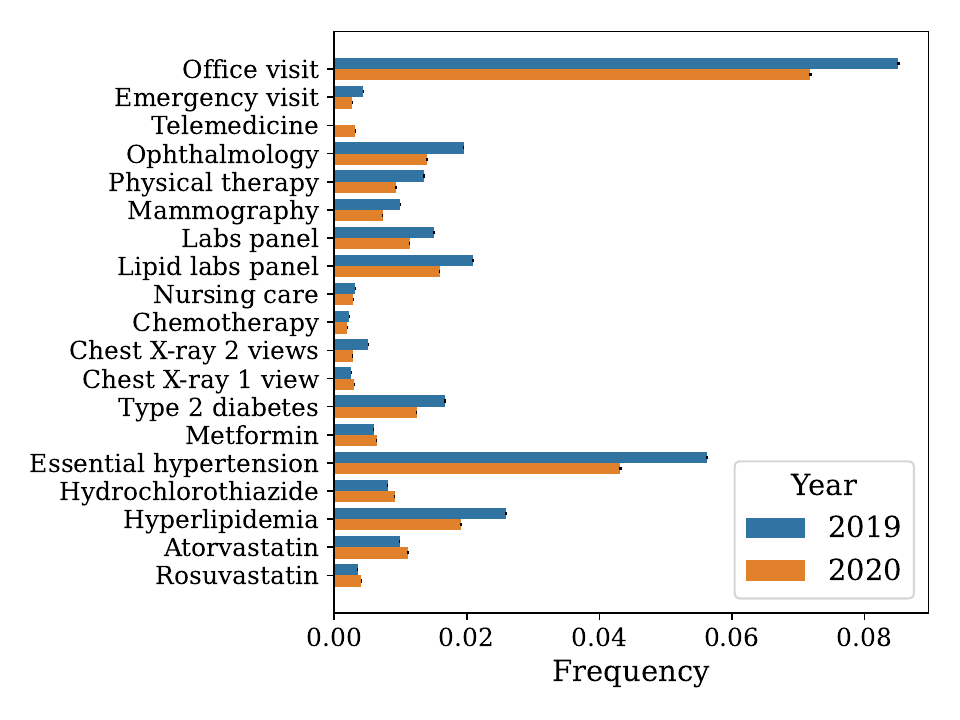}}
    \vspace{-5px}
\end{figure}

\figureref{fig:covariate_shift_2020} shows a selection of shifted features. The frequency of office and emergency department visits dropped in 2020. This observation aligns with how patients were less inclined to visit the doctor for non-COVID related reasons \citep{hartnett2020impact}. Instead, telemedicine became more common \citep{koonin2020trends}. Similarly, routine services, including ophthalmology, physical therapy, screening mammography, and lab panels, decreased significantly. These changes followed US guidelines issued in April 2020 regarding postponing non-urgent services \citep{berkenstock2021changes}. Nursing care also decreased likely due to COVID outbreaks among nursing staff \citep{grabowski2020nursing}. The decrease in chemotherapy procedures reflects how cancer care was delayed during the COVID-19 pandemic \citep{mayo2021cancer}. We see similar trends in the outcomes that are affected by temporal shift and listed in Tables~\ref{tab:condition_results}-\ref{tab:lab_results} in \appendixref{app:scan_results}.

Interestingly, while claims for common conditions decreased, prescription rates for drugs primarily treating those conditions actually increased. These opposing changes suggest the prevalence of chronic illnesses stayed constant but routine clinical visits for managing these conditions decreased. We observe this relation between diabetes and metformin, hypertension and hydrochlorothiazide, and hyperlipidemia and two statins. We discuss domain shift trends further in \appendixref{app:cov_shift}.

\textbf{Challenges to solving the case:} Many methods exist for handling domain shift \citep{mummadi2021test,srivastava2020robustness,chen2020self,zhang2020adaptive,arjovsky2019invariant,sun2016return}. Some approaches seek to reverse the shift by transforming features $X$ to $X'$. Others select different features $X'$. In both cases, the property $X' \perp D \vert U$ is desired. We examine why in Appendix \ref{app:domain_shift_theory}.

Obtaining $X'$ with data from a single source distribution is challenging. To demonstrate, we examine two clinically motivated approaches: selecting features that are less affected by temporal shift and imputing important missing features. We show empirically that these approaches are unable to address domain shift in Appendix \ref{app:empirical_domain_shift}.

\subsection{Conditional Shift in Clinical Procedures}
\label{sec:cond_shift}

Conditional shift is characterized by changes in $\mathbb{P}\left(Y \vert X\right)$ \citep{zhang2013domain}. 

\textbf{Laying out the case:} When comparing logistic regressions from adjacent years, different signs for a feature coefficient may suggest conditional shift with respect to that feature. We obtain 95\% confidence intervals for coefficients using statsmodels \citep{seabold2010statsmodels}. See \appendixref{app:cond_shift_feat_select} for details on feature selection. If the 95\% confidence interval for a coefficient is above 0 for the previous year and below 0 for the current year (or vice versa), the feature is identified as a candidate for conditional shift. 

An example of a non-stationary outcome with conditional shift is receiving an inpatient consultation in 2019. Congestive heart failure and atherosclerosis of coronary artery without angina pectoris became significantly less predictive of the outcome. In general, the frequency of inpatient consultations dropped from .52\% to .35\%. Among patients with congestive heart failure or atherosclerosis, the outcome frequency dropped many times more from 7.5\% to 1.8\% and 3.5\% to 1.2\%, respectively. Medicare reimbursement policies stopped covering inpatient consultations at that time. Since congestive heart failure and atherosclerosis are more prevalent among Medicare patients, more patients with these conditions were affected by the policy change.

Another example of conditional shift is the increased likelihood of receiving nursing care in the next 3 months given discharge from a nursing facility in the past 30 days. While the frequency of nursing care increased from .35\% to .89\% in 2016, among patients who were discharged in the past 30 days, the outcome increased more rapidly from 16\% to 55\%. \appendixref{app:cond_shift_ex_stats} has more details on both examples.

\textbf{Challenges to solving the case:} Addressing conditional shift requires further assumptions on what forms of shift are allowed \citep{zhang2013domain}. \theoremref{thm:cond_shift} in \appendixref{app:cond_shift_recal_proof} provides an example of a set of assumptions. When those assumptions hold, conditional shift can be addressed by selecting a subset of features $C$ and multiplying each predicted probability by $\frac{\mathbb{P}_t\left(Y \vert C\right)}{\mathbb{P}_{t-1}\left(Y \vert C\right)}$. Learning $\mathbb{P}_t\left(Y \vert C\right)$ requires fewer samples from the new time point than fitting a new model since $C$ has fewer features. However, a set of features $C$ that satisfies the assumptions we give is unlikely to exist in healthcare settings.

A set of features $B$ that satisfies the assumption $\mathbb{P}_t\left(Y \vert B\right) = \mathbb{P}_{t-1}\left(Y \vert B\right)$ is more attainable. A model that only uses features in $B$ would not be affected by conditional shift at time $t$. However, predictive information from other features may be lost. We demonstrate these challenges empirically in \appendixref{app:cond_shift_challenge_ex} on the two examples of conditional shift.

\section{Related Work}
\label{sec:related}

Distribution shift has been studied in healthcare settings. The WILDS benchmark includes a dataset of tumor pathology images from different hospitals \citep{koh2021wilds}. \citet{zhang2022adadiag} apply adversarial domain adaptation to adapt a heart failure model across different hospitals. \citet{subbaswamy2020development} and \citet{subbaswamy2021evaluating} identify distribution shifts under which a model may no longer be stable during deployment. \citet{schrouff2022maintaining} examine how fairness may be affected under distribution shifts. We discuss more work on distribution shift outside healthcare in \appendixref{app:related_work}.

Past works have also focused on temporal shifts in healthcare. \citet{guo2021systematic} review 15 studies that mitigate the effects via model refitting, online learning, ensembling, or model selection. \citet{nestor2019feature} show aggregating features into expert-defined clinical concepts increases robustness to temporal shift. \citet{guo2023ehr} show using foundation models also improves robustness. \citet{jung2015implications} examine temporal shift by comparing the performance drop when splitting by time period versus by patient or randomly. \citet{saez2020ehrtemporalvariability} develop a package to identify temporal variability in features. The Wild-Time benchmark evaluates shift across 4 time periods in MIMIC-IV using the definition in our baseline \citep{yao2022wild,johnson2020mimic}. \citet{otles2021mind} attribute deterioration as measured by our baseline definition to a combination of temporal and infrastructure shift. \citet{guo2022evaluation} perform the comparison in our definition for 4 outcomes in MIMIC-IV. We examine temporal shift across a much larger number of tasks. \citet{zhou2022model} scan for temporal shifts in a single outcome over long time ranges by identifying when using stale models results in a large performance drop. Whereas their approach scans for shift across the temporal dimension, our algorithm can scan for temporal shift across three dimensions: time, outcome, and sub-population.

\section{Discussion}

In this work, we propose an algorithm to test whether temporal shift affects an outcome at a particular time either at the population level or within a sub-population. Then, we integrate our method into a meta-algorithm to scan for temporal shifts across a large number of outcomes and time points. Finally, we demonstrate how our approach unearths examples of temporal shifts in a large health insurance claims dataset. We provide code and guidelines so that our study can be generalized to other datasets and tasks.

A limitation of our work is that our algorithms detect temporal shift in a retrospective manner. A prospective approach would detect shift as predictions are made. Prospective tests are challenging. Assessing conditional and label shift depends on the outcome, and outcomes may not be available for a period of time after predictions are made.

We also propose two directions for future work. First, since we demonstrate temporal shifts may affect only sub-populations, we recommend developing additional subgroup discovery methods to identify sub-populations to monitor for temporal shift. Second, once a shift is detected, updating the model with limited data from a new time point is critical for mitigating the effects of temporal shifts on patient care.

\acks{We are grateful to Aaron Smith-McLallen, Stephanie Gervasi, James Denyer, Eric Wilkinson, George Fenimore, Monique Wilkins, and the rest of the data science group at Independence Blue Cross, whose expertise, data, and financial support enabled this study. We would like to thank Rebecca Boiarsky, Alejandro Buendia, and Hunter Lang for help with the database and extraction pipeline. We are grateful to Michael Oberst and Yury Polyanskiy for advice on hypothesis testing and confidence intervals. We also value the clinical insights from Leora Horwitz and Saul Blecker. This draft was improved by suggestions from Hussein Mozannar and our anonymous reviewers.}

\bibliography{main}

\begin{thebibliography}{71}
\providecommand{\natexlab}[1]{#1}
\providecommand{\url}[1]{\texttt{#1}}
\expandafter\ifx\csname urlstyle\endcsname\relax
  \providecommand{\doi}[1]{doi: #1}\else
  \providecommand{\doi}{doi: \begingroup \urlstyle{rm}\Url}\fi

\bibitem[Arjovsky et~al.(2019)Arjovsky, Bottou, Gulrajani, and
  Lopez-Paz]{arjovsky2019invariant}
Martin Arjovsky, L{\'e}on Bottou, Ishaan Gulrajani, and David Lopez-Paz.
\newblock Invariant risk minimization.
\newblock \emph{arXiv preprint arXiv:1907.02893}, 2019.

\bibitem[Association(2022)]{american2022introduction}
American~Diabetes Association.
\newblock Introduction: Standards of medical care in diabetes—2022, 2022.

\bibitem[Bandos et~al.(2005)Bandos, Rockette, and Gur]{bandos2005permutation}
Andriy~I Bandos, Howard~E Rockette, and David Gur.
\newblock A permutation test sensitive to differences in areas for comparing
  {ROC} curves from a paired design.
\newblock \emph{Statistics in medicine}, 24\penalty0 (18):\penalty0 2873--2893,
  2005.

\bibitem[Benjamini and Hochberg(1995)]{benjamini1995controlling}
Yoav Benjamini and Yosef Hochberg.
\newblock Controlling the false discovery rate: a practical and powerful
  approach to multiple testing.
\newblock \emph{Journal of the Royal statistical society: series B
  (Methodological)}, 57\penalty0 (1):\penalty0 289--300, 1995.

\bibitem[Berkenstock et~al.(2021)Berkenstock, Liberman, McDonnell, and
  Chaon]{berkenstock2021changes}
Meghan~K Berkenstock, Paulina Liberman, Peter~J McDonnell, and Benjamin~C
  Chaon.
\newblock Changes in patient visits and diagnoses in a large academic center
  during the {COVID-19} pandemic.
\newblock \emph{BMC ophthalmology}, 21\penalty0 (1):\penalty0 1--9, 2021.

\bibitem[Cao and Slobounov(2011)]{cao2011application}
Cheng Cao and Semyon Slobounov.
\newblock Application of a novel measure of eeg non-stationarity as
  ‘{Shannon}-entropy of the peak frequency shifting’ for detecting residual
  abnormalities in concussed individuals.
\newblock \emph{Clinical Neurophysiology}, 122\penalty0 (7):\penalty0
  1314--1321, 2011.

\bibitem[Chandola et~al.(2009)Chandola, Banerjee, and
  Kumar]{chandola2009anomaly}
Varun Chandola, Arindam Banerjee, and Vipin Kumar.
\newblock Anomaly detection: A survey.
\newblock \emph{ACM computing surveys (CSUR)}, 41\penalty0 (3):\penalty0 1--58,
  2009.

\bibitem[Chen et~al.(2020)Chen, Wei, Kumar, and Ma]{chen2020self}
Yining Chen, Colin Wei, Ananya Kumar, and Tengyu Ma.
\newblock Self-training avoids using spurious features under domain shift.
\newblock \emph{Advances in Neural Information Processing Systems},
  33:\penalty0 21061--21071, 2020.

\bibitem[Cohen et~al.(2021)Cohen, Cao, Viviano, Huang, Fralick, Ghassemi,
  Mamdani, Greiner, and Bengio]{cohen2021problems}
Joseph~Paul Cohen, Tianshi Cao, Joseph~D Viviano, Chin-Wei Huang, Michael
  Fralick, Marzyeh Ghassemi, Muhammad Mamdani, Russell Greiner, and Yoshua
  Bengio.
\newblock Problems in the deployment of machine-learned models in health care.
\newblock \emph{CMAJ}, 193\penalty0 (35):\penalty0 E1391--E1394, 2021.

\bibitem[Finlayson et~al.(2021)Finlayson, Subbaswamy, Singh, Bowers, Kupke,
  Zittrain, Kohane, and Saria]{finlayson2021clinician}
Samuel~G Finlayson, Adarsh Subbaswamy, Karandeep Singh, John Bowers, Annabel
  Kupke, Jonathan Zittrain, Isaac~S Kohane, and Suchi Saria.
\newblock The clinician and dataset shift in artificial intelligence.
\newblock \emph{The New England journal of medicine}, 385\penalty0
  (3):\penalty0 283, 2021.

\bibitem[Food et~al.(2019)Food, Administration, et~al.]{us2019artificial}
US~Food, Drug Administration, et~al.
\newblock Artificial intelligence and machine learning in software as a medical
  device.
\newblock \emph{Silverspring: US Food and Drug Administration}, 2019.

\bibitem[Grabowski and Mor(2020)]{grabowski2020nursing}
David~C Grabowski and Vincent Mor.
\newblock Nursing home care in crisis in the wake of {COVID-19}.
\newblock \emph{Jama}, 324\penalty0 (1):\penalty0 23--24, 2020.

\bibitem[Guo et~al.(2020)Guo, Gong, Liu, Zhang, and Tao]{guo2020ltf}
Jiaxian Guo, Mingming Gong, Tongliang Liu, Kun Zhang, and Dacheng Tao.
\newblock Ltf: A label transformation framework for correcting label shift.
\newblock In \emph{International Conference on Machine Learning}, pages
  3843--3853. PMLR, 2020.

\bibitem[Guo et~al.(2021)Guo, Pfohl, Fries, Posada, Fleming, Aftandilian, Shah,
  and Sung]{guo2021systematic}
Lin~Lawrence Guo, Stephen~R Pfohl, Jason Fries, Jose Posada, Scott~Lanyon
  Fleming, Catherine Aftandilian, Nigam Shah, and Lillian Sung.
\newblock Systematic review of approaches to preserve machine learning
  performance in the presence of temporal dataset shift in clinical medicine.
\newblock \emph{Applied clinical informatics}, 12\penalty0 (04):\penalty0
  808--815, 2021.

\bibitem[Guo et~al.(2022)Guo, Pfohl, Fries, Johnson, Posada, Aftandilian, Shah,
  and Sung]{guo2022evaluation}
Lin~Lawrence Guo, Stephen~R Pfohl, Jason Fries, Alistair~EW Johnson, Jose
  Posada, Catherine Aftandilian, Nigam Shah, and Lillian Sung.
\newblock Evaluation of domain generalization and adaptation on improving model
  robustness to temporal dataset shift in clinical medicine.
\newblock \emph{Scientific reports}, 12\penalty0 (1):\penalty0 1--10, 2022.

\bibitem[Guo et~al.(2023)Guo, Steinberg, Fleming, Posada, Lemmon, Pfohl, Shah,
  Fries, and Sung]{guo2023ehr}
Lin~Lawrence Guo, Ethan Steinberg, Scott~Lanyon Fleming, Jose Posada, Joshua
  Lemmon, Stephen~R Pfohl, Nigam Shah, Jason Fries, and Lillian Sung.
\newblock Ehr foundation models improve robustness in the presence of temporal
  distribution shift.
\newblock \emph{Scientific Reports}, 13\penalty0 (1):\penalty0 3767, 2023.

\bibitem[Hartnett et~al.(2020)Hartnett, Kite-Powell, DeVies, Coletta, Boehmer,
  Adjemian, Gundlapalli, et~al.]{hartnett2020impact}
Kathleen~P Hartnett, Aaron Kite-Powell, Jourdan DeVies, Michael~A Coletta,
  Tegan~K Boehmer, Jennifer Adjemian, Adi~V Gundlapalli, et~al.
\newblock Impact of the {COVID-19} pandemic on emergency department
  visits--{United States, January 1, 2019}--{May 30, 2020}.
\newblock \emph{Morbidity and Mortality Weekly Report}, 69\penalty0
  (23):\penalty0 699, 2020.

\bibitem[Hendrycks and Gimpel(2016)]{hendrycks2016baseline}
Dan Hendrycks and Kevin Gimpel.
\newblock A baseline for detecting misclassified and out-of-distribution
  examples in neural networks.
\newblock \emph{arXiv preprint arXiv:1610.02136}, 2016.

\bibitem[Hripcsak et~al.(2015)Hripcsak, Duke, Shah, Reich, Huser, Schuemie,
  Suchard, Park, Wong, Rijnbeek, et~al.]{hripcsak2015observational}
George Hripcsak, Jon~D Duke, Nigam~H Shah, Christian~G Reich, Vojtech Huser,
  Martijn~J Schuemie, Marc~A Suchard, Rae~Woong Park, Ian Chi~Kei Wong, Peter~R
  Rijnbeek, et~al.
\newblock {Observational Health Data Sciences and Informatics (OHDSI)}:
  opportunities for observational researchers.
\newblock In \emph{MEDINFO 2015: eHealth-enabled Health}, pages 574--578. IOS
  Press, 2015.

\bibitem[Jacobi et~al.(2020)Jacobi, Chung, Bernheim, and
  Eber]{jacobi2020portable}
Adam Jacobi, Michael Chung, Adam Bernheim, and Corey Eber.
\newblock Portable chest {X-ray} in coronavirus disease-19 {(COVID-19)}: A
  pictorial review.
\newblock \emph{Clinical imaging}, 64:\penalty0 35--42, 2020.

\bibitem[Johnson et~al.(2020)Johnson, Bulgarelli, Pollard, Horng, Celi, and
  Mark]{johnson2020mimic}
Alistair Johnson, Lucas Bulgarelli, Tom Pollard, Steven Horng, Leo~Anthony
  Celi, and Roger Mark.
\newblock {MIMIC-IV}.
\newblock \emph{PhysioNet. Available online at: https://physionet.
  org/content/mimiciv/1.0/(accessed August 23, 2021)}, 2020.

\bibitem[Jung and Shah(2015)]{jung2015implications}
Kenneth Jung and Nigam~H Shah.
\newblock Implications of non-stationarity on predictive modeling using ehrs.
\newblock \emph{Journal of biomedical informatics}, 58:\penalty0 168--174,
  2015.

\bibitem[Kamath et~al.(2021)Kamath, Tangella, Sutherland, and
  Srebro]{kamath2021does}
Pritish Kamath, Akilesh Tangella, Danica Sutherland, and Nathan Srebro.
\newblock Does invariant risk minimization capture invariance?
\newblock In \emph{International Conference on Artificial Intelligence and
  Statistics}, pages 4069--4077. PMLR, 2021.

\bibitem[Kodialam et~al.(2021)Kodialam, Boiarsky, Lim, Sai, Dixit, and
  Sontag]{kodialam2021deep}
Rohan Kodialam, Rebecca Boiarsky, Justin Lim, Aditya Sai, Neil Dixit, and David
  Sontag.
\newblock Deep contextual clinical prediction with reverse distillation.
\newblock \emph{Proceedings of the AAAI Conference on Artificial Intelligence},
  35\penalty0 (1):\penalty0 249--258, 2021.

\bibitem[Koh et~al.(2021)Koh, Sagawa, Marklund, Xie, Zhang, Balsubramani, Hu,
  Yasunaga, Phillips, Gao, et~al.]{koh2021wilds}
Pang~Wei Koh, Shiori Sagawa, Henrik Marklund, Sang~Michael Xie, Marvin Zhang,
  Akshay Balsubramani, Weihua Hu, Michihiro Yasunaga, Richard~Lanas Phillips,
  Irena Gao, et~al.
\newblock {WILDS}: A benchmark of in-the-wild distribution shifts.
\newblock In \emph{International Conference on Machine Learning}, pages
  5637--5664. PMLR, 2021.

\bibitem[Koonin et~al.(2020)Koonin, Hoots, Tsang, Leroy, Farris, Jolly, Antall,
  McCabe, Zelis, Tong, et~al.]{koonin2020trends}
Lisa~M Koonin, Brooke Hoots, Clarisse~A Tsang, Zanie Leroy, Kevin Farris,
  Brandon Jolly, Peter Antall, Bridget McCabe, Cynthia~BR Zelis, Ian Tong,
  et~al.
\newblock Trends in the use of telehealth during the emergence of the
  {COVID-19} pandemic--{United States, January}--{March} 2020.
\newblock \emph{Morbidity and Mortality Weekly Report}, 69\penalty0
  (43):\penalty0 1595, 2020.

\bibitem[Krishnamurthy et~al.(2021)Krishnamurthy, Ks, Dovgan, Lu{\v{s}}trek,
  Gradi{\v{s}}ek~Pileti{\v{c}}, Srinivasan, Li, Gradi{\v{s}}ek, and
  Syed-Abdul]{krishnamurthy2021machine}
Surya Krishnamurthy, Kapeleshh Ks, Erik Dovgan, Mitja Lu{\v{s}}trek, Barbara
  Gradi{\v{s}}ek~Pileti{\v{c}}, Kathiravan Srinivasan, Yu-Chuan Li, Anton
  Gradi{\v{s}}ek, and Shabbir Syed-Abdul.
\newblock Machine learning prediction models for chronic kidney disease using
  national health insurance claim data in taiwan.
\newblock \emph{Healthcare}, 9\penalty0 (5):\penalty0 546, 2021.

\bibitem[Levey et~al.(2006)Levey, Coresh, Greene, Stevens, Zhang, Hendriksen,
  Kusek, Van~Lente, and Collaboration*]{levey2006using}
Andrew~S Levey, Josef Coresh, Tom Greene, Lesley~A Stevens, Yaping Zhang,
  Stephen Hendriksen, John~W Kusek, Frederick Van~Lente, and Chronic Kidney
  Disease~Epidemiology Collaboration*.
\newblock Using standardized serum creatinine values in the modification of
  diet in renal disease study equation for estimating glomerular filtration
  rate.
\newblock \emph{Annals of internal medicine}, 145\penalty0 (4):\penalty0
  247--254, 2006.

\bibitem[Levey et~al.(2009)Levey, Stevens, Schmid, Zhang, Castro~III, Feldman,
  Kusek, Eggers, Van~Lente, Greene, et~al.]{levey2009new}
Andrew~S Levey, Lesley~A Stevens, Christopher~H Schmid, Yaping Zhang,
  Alejandro~F Castro~III, Harold~I Feldman, John~W Kusek, Paul Eggers,
  Frederick Van~Lente, Tom Greene, et~al.
\newblock A new equation to estimate glomerular filtration rate.
\newblock \emph{Annals of internal medicine}, 150\penalty0 (9):\penalty0
  604--612, 2009.

\bibitem[Lipton et~al.(2018)Lipton, Wang, and Smola]{lipton2018detecting}
Zachary Lipton, Yu-Xiang Wang, and Alexander Smola.
\newblock Detecting and correcting for label shift with black box predictors.
\newblock In \emph{International conference on machine learning}, pages
  3122--3130. PMLR, 2018.

\bibitem[Mayo et~al.(2021)Mayo, Potugari, Bzeih, Scheidel, Carrera, and
  Shellenberger]{mayo2021cancer}
MacKenzie Mayo, Bindu Potugari, Rami Bzeih, Caleb Scheidel, Carolyn Carrera,
  and Richard~A Shellenberger.
\newblock Cancer screening during the {COVID-19} pandemic: A systematic review
  and meta-analysis.
\newblock \emph{Mayo Clinic Proceedings: Innovations, Quality \& Outcomes},
  5\penalty0 (6):\penalty0 1109--1117, 2021.

\bibitem[Mizuno et~al.(2021)Mizuno, Patel, Park, Hare, Harrington, and
  Adusumalli]{mizuno2021statin}
Atsushi Mizuno, Mitesh~S Patel, Sae-Hwan Park, Allison~J Hare, Tory~O
  Harrington, and Srinath Adusumalli.
\newblock Statin prescribing patterns during in-person and telemedicine visits
  before and during the {COVID-19} pandemic.
\newblock \emph{Circulation: Cardiovascular Quality and Outcomes}, 14\penalty0
  (10):\penalty0 e008266, 2021.

\bibitem[Mummadi et~al.(2021)Mummadi, Hutmacher, Rambach, Levinkov, Brox, and
  Metzen]{mummadi2021test}
Chaithanya~Kumar Mummadi, Robin Hutmacher, Kilian Rambach, Evgeny Levinkov,
  Thomas Brox, and Jan~Hendrik Metzen.
\newblock Test-time adaptation to distribution shift by confidence maximization
  and input transformation.
\newblock \emph{arXiv preprint arXiv:2106.14999}, 2021.

\bibitem[Nason(2006)]{nason2006stationary}
Guy~P Nason.
\newblock Stationary and non-stationary time series.
\newblock \emph{Statistics in volcanology}, 60, 2006.

\bibitem[Nestor et~al.(2019)Nestor, McDermott, Boag, Berner, Naumann, Hughes,
  Goldenberg, and Ghassemi]{nestor2019feature}
Bret Nestor, Matthew~BA McDermott, Willie Boag, Gabriela Berner, Tristan
  Naumann, Michael~C Hughes, Anna Goldenberg, and Marzyeh Ghassemi.
\newblock Feature robustness in non-stationary health records: caveats to
  deployable model performance in common clinical machine learning tasks.
\newblock In \emph{Machine Learning for Healthcare Conference}, pages 381--405.
  PMLR, 2019.

\bibitem[Otles et~al.(2021)Otles, Oh, Li, Bochinski, Joo, Ortwine, Shenoy,
  Washer, Young, Rao, et~al.]{otles2021mind}
Erkin Otles, Jeeheh Oh, Benjamin Li, Michelle Bochinski, Hyeon Joo, Justin
  Ortwine, Erica Shenoy, Laraine Washer, Vincent~B Young, Krishna Rao, et~al.
\newblock Mind the performance gap: examining dataset shift during prospective
  validation.
\newblock In \emph{Machine Learning for Healthcare Conference}, pages 506--534.
  PMLR, 2021.

\bibitem[Outcomes and Group(2013)]{kidney2013kdigo}
Kidney Disease: Improving~Global Outcomes and CKD~Work Group.
\newblock {KDIGO} 2012 clinical practice guideline for the evaluation and
  management of chronic kidney disease.
\newblock \emph{Kidney Int}, 3\penalty0 (1):\penalty0 1--150, 2013.

\bibitem[Pedregosa et~al.(2011)Pedregosa, Varoquaux, Gramfort, Michel, Thirion,
  Grisel, Blondel, Prettenhofer, Weiss, Dubourg, et~al.]{pedregosa2011scikit}
Fabian Pedregosa, Ga{\"e}l Varoquaux, Alexandre Gramfort, Vincent Michel,
  Bertrand Thirion, Olivier Grisel, Mathieu Blondel, Peter Prettenhofer, Ron
  Weiss, Vincent Dubourg, et~al.
\newblock Scikit-learn: Machine learning in python.
\newblock \emph{the Journal of machine Learning research}, 12:\penalty0
  2825--2830, 2011.

\bibitem[Phipson and Smyth(2010)]{phipson2010permutation}
Belinda Phipson and Gordon~K Smyth.
\newblock Permutation p-values should never be zero: calculating exact p-values
  when permutations are randomly drawn.
\newblock \emph{Statistical applications in genetics and molecular biology},
  9\penalty0 (1), 2010.

\bibitem[Quinonero-Candela et~al.(2008)Quinonero-Candela, Sugiyama,
  Schwaighofer, and Lawrence]{quinonero2008dataset}
Joaquin Quinonero-Candela, Masashi Sugiyama, Anton Schwaighofer, and Neil~D
  Lawrence.
\newblock \emph{Dataset shift in machine learning}.
\newblock Mit Press, 2008.

\bibitem[Razavian et~al.(2015)Razavian, Blecker, Schmidt, Smith-McLallen,
  Nigam, and Sontag]{razavian2015population}
Narges Razavian, Saul Blecker, Ann~Marie Schmidt, Aaron Smith-McLallen, Somesh
  Nigam, and David Sontag.
\newblock Population-level prediction of type 2 diabetes from claims data and
  analysis of risk factors.
\newblock \emph{Big Data}, 3\penalty0 (4):\penalty0 277--287, 2015.

\bibitem[Robin et~al.(2011)Robin, Turck, Hainard, Tiberti, Lisacek, Sanchez,
  and M{\"u}ller]{robin2011proc}
Xavier Robin, Natacha Turck, Alexandre Hainard, Natalia Tiberti,
  Fr{\'e}d{\'e}rique Lisacek, Jean-Charles Sanchez, and Markus M{\"u}ller.
\newblock {pROC}: an open-source package for {R and S+} to analyze and compare
  {ROC} curves.
\newblock \emph{BMC bioinformatics}, 12\penalty0 (1):\penalty0 1--8, 2011.

\bibitem[Rosenfeld et~al.(2020)Rosenfeld, Ravikumar, and
  Risteski]{rosenfeld2020risks}
Elan Rosenfeld, Pradeep Ravikumar, and Andrej Risteski.
\newblock The risks of invariant risk minimization.
\newblock \emph{arXiv preprint arXiv:2010.05761}, 2020.

\bibitem[S{\'a}ez et~al.(2020)S{\'a}ez, Guti{\'e}rrez-Sacrist{\'a}n, Kohane,
  Garc{\'\i}a-G{\'o}mez, and Avillach]{saez2020ehrtemporalvariability}
Carlos S{\'a}ez, Alba Guti{\'e}rrez-Sacrist{\'a}n, Isaac Kohane, Juan~M
  Garc{\'\i}a-G{\'o}mez, and Paul Avillach.
\newblock Ehrtemporalvariability: delineating temporal data-set shifts in
  electronic health records.
\newblock \emph{Gigascience}, 9\penalty0 (8):\penalty0 giaa079, 2020.

\bibitem[Sagawa et~al.(2019)Sagawa, Koh, Hashimoto, and
  Liang]{sagawa2019distributionally}
Shiori Sagawa, Pang~Wei Koh, Tatsunori~B Hashimoto, and Percy Liang.
\newblock Distributionally robust neural networks for group shifts: On the
  importance of regularization for worst-case generalization.
\newblock \emph{arXiv preprint arXiv:1911.08731}, 2019.

\bibitem[Salem et~al.(2014)Salem, Liu, Mehaoua, and Boutaba]{salem2014online}
Osman Salem, Yaning Liu, Ahmed Mehaoua, and Raouf Boutaba.
\newblock Online anomaly detection in wireless body area networks for reliable
  healthcare monitoring.
\newblock \emph{IEEE journal of biomedical and health informatics}, 18\penalty0
  (5):\penalty0 1541--1551, 2014.

\bibitem[Sankararaman et~al.(2022)Sankararaman, Narayanaswamy, Singh, and
  Song]{sankararaman2022fitness}
Abishek Sankararaman, Balakrishnan Narayanaswamy, Vikramank~Y Singh, and Zhao
  Song.
\newblock {FITNESS: (Fine Tune on New and Similar Samples)} to detect anomalies
  in streams with drift and outliers.
\newblock In \emph{International Conference on Machine Learning}, pages
  19153--19177. PMLR, 2022.

\bibitem[Santurkar et~al.(2020)Santurkar, Tsipras, and
  Madry]{santurkar2020breeds}
Shibani Santurkar, Dimitris Tsipras, and Aleksander Madry.
\newblock Breeds: Benchmarks for subpopulation shift.
\newblock \emph{arXiv preprint arXiv:2008.04859}, 2020.

\bibitem[Schrouff et~al.(2022)Schrouff, Harris, Koyejo, Alabdulmohsin,
  Schnider, Opsahl-Ong, Brown, Roy, Mincu, Chen,
  et~al.]{schrouff2022maintaining}
Jessica Schrouff, Natalie Harris, Oluwasanmi Koyejo, Ibrahim Alabdulmohsin, Eva
  Schnider, Krista Opsahl-Ong, Alex Brown, Subhrajit Roy, Diana Mincu,
  Christina Chen, et~al.
\newblock Maintaining fairness across distribution shift: do we have viable
  solutions for real-world applications?
\newblock \emph{arXiv preprint arXiv:2202.01034}, 2022.

\bibitem[Seabold and Perktold(2010)]{seabold2010statsmodels}
Skipper Seabold and Josef Perktold.
\newblock statsmodels: Econometric and statistical modeling with python.
\newblock In \emph{9th Python in Science Conference}, 2010.

\bibitem[Segal et~al.(2020)Segal, Kalifa, Radinsky, Ehrenberg, Elad, Maor,
  Lewis, Tibi, Korn, and Koren]{segal2020machine}
Zvi Segal, Dan Kalifa, Kira Radinsky, Bar Ehrenberg, Guy Elad, Gal Maor, Maor
  Lewis, Muhammad Tibi, Liat Korn, and Gideon Koren.
\newblock Machine learning algorithm for early detection of end-stage renal
  disease.
\newblock \emph{BMC nephrology}, 21\penalty0 (1):\penalty0 1--10, 2020.

\bibitem[Sendak et~al.(2020)Sendak, Ratliff, Sarro, Alderton, Futoma, Gao,
  Nichols, Revoir, Yashar, Miller, et~al.]{sendak2020real}
Mark~P Sendak, William Ratliff, Dina Sarro, Elizabeth Alderton, Joseph Futoma,
  Michael Gao, Marshall Nichols, Mike Revoir, Faraz Yashar, Corinne Miller,
  et~al.
\newblock Real-world integration of a sepsis deep learning technology into
  routine clinical care: implementation study.
\newblock \emph{JMIR medical informatics}, 8\penalty0 (7):\penalty0 e15182,
  2020.

\bibitem[Shetty et~al.(2021)Shetty, Imas, and Pursnani]{shetty2021effect}
M~Shetty, P~Imas, and A~Pursnani.
\newblock Effect of the {COVID-19} pandemic on coronary artery calcium testing
  and subsequent statin prescription.
\newblock \emph{Journal of Cardiovascular Computed Tomography}, 15\penalty0
  (4):\penalty0 S33--S34, 2021.

\bibitem[Song et~al.(2015)Song, Flach, and Kalogridis]{song2015dataset}
Hao Song, Peter Flach, and Georgios Kalogridis.
\newblock Dataset shift detection with model-based subgroup discovery.
\newblock In \emph{International Workshop on Learning over Multiple Contexts
  (LMCE)}, 2015.

\bibitem[Srivastava et~al.(2020)Srivastava, Hashimoto, and
  Liang]{srivastava2020robustness}
Megha Srivastava, Tatsunori Hashimoto, and Percy Liang.
\newblock Robustness to spurious correlations via human annotations.
\newblock In \emph{International Conference on Machine Learning}, pages
  9109--9119. PMLR, 2020.

\bibitem[Subbaswamy and Saria(2020)]{subbaswamy2020development}
Adarsh Subbaswamy and Suchi Saria.
\newblock From development to deployment: dataset shift, causality, and
  shift-stable models in health ai.
\newblock \emph{Biostatistics}, 21\penalty0 (2):\penalty0 345--352, 2020.

\bibitem[Subbaswamy et~al.(2021)Subbaswamy, Adams, and
  Saria]{subbaswamy2021evaluating}
Adarsh Subbaswamy, Roy Adams, and Suchi Saria.
\newblock Evaluating model robustness and stability to dataset shift.
\newblock In \emph{International Conference on Artificial Intelligence and
  Statistics}, pages 2611--2619. PMLR, 2021.

\bibitem[Sun et~al.(2016)Sun, Feng, and Saenko]{sun2016return}
Baochen Sun, Jiashi Feng, and Kate Saenko.
\newblock Return of frustratingly easy domain adaptation.
\newblock In \emph{Proceedings of the AAAI conference on artificial
  intelligence}, volume~30, 2016.

\bibitem[Vlahogianni et~al.(2006)Vlahogianni, Karlaftis, and
  Golias]{vlahogianni2006statistical}
Eleni~I Vlahogianni, Matthew~G Karlaftis, and John~C Golias.
\newblock Statistical methods for detecting nonlinearity and non-stationarity
  in univariate short-term time-series of traffic volume.
\newblock \emph{Transportation Research Part C: Emerging Technologies},
  14\penalty0 (5):\penalty0 351--367, 2006.

\bibitem[Wang et~al.(2020)Wang, Shelhamer, Liu, Olshausen, and
  Darrell]{wang2020tent}
Dequan Wang, Evan Shelhamer, Shaoteng Liu, Bruno Olshausen, and Trevor Darrell.
\newblock Tent: Fully test-time adaptation by entropy minimization.
\newblock \emph{arXiv preprint arXiv:2006.10726}, 2020.

\bibitem[Westra and Sisson(2011)]{westra2011detection}
Seth Westra and Scott~A Sisson.
\newblock Detection of non-stationarity in precipitation extremes using a
  max-stable process model.
\newblock \emph{Journal of Hydrology}, 406\penalty0 (1-2):\penalty0 119--128,
  2011.

\bibitem[Wiens et~al.(2019)Wiens, Saria, Sendak, Ghassemi, Liu, Doshi-Velez,
  Jung, Heller, Kale, Saeed, et~al.]{wiens2019no}
Jenna Wiens, Suchi Saria, Mark Sendak, Marzyeh Ghassemi, Vincent~X Liu, Finale
  Doshi-Velez, Kenneth Jung, Katherine Heller, David Kale, Mohammed Saeed,
  et~al.
\newblock Do no harm: a roadmap for responsible machine learning for health
  care.
\newblock \emph{Nature medicine}, 25\penalty0 (9):\penalty0 1337--1340, 2019.

\bibitem[Wiles et~al.(2021)Wiles, Gowal, Stimberg, Alvise-Rebuffi, Ktena,
  Dvijotham, and Cemgil]{wiles2021fine}
Olivia Wiles, Sven Gowal, Florian Stimberg, Sylvestre Alvise-Rebuffi, Ira
  Ktena, Krishnamurthy Dvijotham, and Taylan Cemgil.
\newblock A fine-grained analysis on distribution shift.
\newblock \emph{arXiv preprint arXiv:2110.11328}, 2021.

\bibitem[Yamanishi and Takeuchi(2002)]{yamanishi2002unifying}
Kenji Yamanishi and Jun-ichi Takeuchi.
\newblock A unifying framework for detecting outliers and change points from
  non-stationary time series data.
\newblock In \emph{Proceedings of the eighth ACM SIGKDD international
  conference on Knowledge discovery and data mining}, pages 676--681, 2002.

\bibitem[Yang et~al.(2022)Yang, Wang, Zou, Zhou, Ding, Peng, Wang, Chen, Li,
  Sun, et~al.]{yang2022openood}
Jingkang Yang, Pengyun Wang, Dejian Zou, Zitang Zhou, Kunyuan Ding, Wenxuan
  Peng, Haoqi Wang, Guangyao Chen, Bo~Li, Yiyou Sun, et~al.
\newblock Openood: Benchmarking generalized out-of-distribution detection.
\newblock \emph{arXiv preprint arXiv:2210.07242}, 2022.

\bibitem[Yao et~al.(2022)Yao, Choi, Cao, Lee, Koh, and Finn]{yao2022wild}
Huaxiu Yao, Caroline Choi, Bochuan Cao, Yoonho Lee, Pang~Wei Koh, and Chelsea
  Finn.
\newblock Wild-time: A benchmark of in-the-wild distribution shift over time.
\newblock In \emph{Thirty-sixth Conference on Neural Information Processing
  Systems Datasets and Benchmarks Track}, 2022.

\bibitem[Zhang et~al.(2013)Zhang, Sch{\"o}lkopf, Muandet, and
  Wang]{zhang2013domain}
Kun Zhang, Bernhard Sch{\"o}lkopf, Krikamol Muandet, and Zhikun Wang.
\newblock Domain adaptation under target and conditional shift.
\newblock In \emph{International conference on machine learning}, pages
  819--827. PMLR, 2013.

\bibitem[Zhang et~al.(2020)Zhang, Marklund, Gupta, Levine, and
  Finn]{zhang2020adaptive}
Marvin Zhang, Henrik Marklund, Abhishek Gupta, Sergey Levine, and Chelsea Finn.
\newblock Adaptive risk minimization: A meta-learning approach for tackling
  group shift.
\newblock \emph{arXiv preprint arXiv:2007.02931}, 8:\penalty0 9, 2020.

\bibitem[Zhang et~al.(2022)Zhang, Chen, and Bui]{zhang2022adadiag}
Tianran Zhang, Muhao Chen, and Alex~AT Bui.
\newblock Adadiag: Adversarial domain adaptation of diagnostic prediction with
  clinical event sequences.
\newblock \emph{Journal of biomedical informatics}, 134:\penalty0 104168, 2022.

\bibitem[Zhou et~al.(2022{\natexlab{a}})Zhou, Balakrishnan, and
  Lipton]{zhou2022domain}
Helen Zhou, Sivaraman Balakrishnan, and Zachary~C Lipton.
\newblock Domain adaptation under missingness shift.
\newblock \emph{arXiv preprint arXiv:2211.02093}, 2022{\natexlab{a}}.

\bibitem[Zhou et~al.(2022{\natexlab{b}})Zhou, Chen, and Lipton]{zhou2022model}
Helen Zhou, Yuwen Chen, and Zachary~C Lipton.
\newblock Model evaluation in medical datasets over time.
\newblock \emph{arXiv preprint arXiv:2211.07165}, 2022{\natexlab{b}}.

\end{thebibliography}

\clearpage
\appendix

\section{Outcome Model Details}

\subsection{Outcome Model Class}
\label{app:outcome_model}

We train the outcome models using scikit-learn \citep{pedregosa2011scikit}. When learning the logistic regressions for each year, samples are weighted to address class imbalance. We use an lbfgs solver with the default scikit-learn tolerance 1e-4. We increase the number of iterations to 1000 due to warnings that the solver did not converge. Multi-collinearity may have been the reason these warnings continued to appear sometimes, but the models are reasonable. We tune the L2 regularization constant on a log scale from 1e-5 to 10 and select the model with the best validation AUC. 

We consider using logistic regressions, decision trees, random forests, or gradient-boosted decision trees for the outcome models. When comparing test AUCs on the top 5 condition outcomes shown in \tableref{tab:outcome_model_classes}, we see logistic regressions, random forests, and gradient-boosted trees have similar performance, that is, the mean AUC for logistic regressions is within 1 standard deviation of the mean AUC for the other two model classes. Because logistic regressions are much faster to train and easier to interpret, we use logistic regressions for the outcome models. To improve efficiency, we train one model per year for each outcome and then assess temporal shift. Thus, the same outcome model for time $t$ is used when evaluating shift at time $t$ and time $t+1$.

\begin{table}[b]
\floatconts
  {tab:outcome_model_classes}%
  {\caption{Performance of each model class for outcome models. Mean and standard deviation in parentheses for test AUC are computed over models from all years for predicting the top 5 condition outcomes. GB: Gradient-boosted.}}%
  {\begin{tabular}{lc}
  \toprule
  \bfseries Model class & \bfseries Test AUC \\
  \midrule
  Logistic reg. & 0.667 (0.034) \\
  Decision tree & 0.577 (0.022) \\
  Random forest & 0.674 (0.029) \\
  GB trees & 0.670 (0.029) \\
  \bottomrule
  \end{tabular}}
\end{table}

\subsection{Choice of Feature Windows}
\label{app:feature_window_choice}
As stated in \sectionref{sec:scan_setup_details}, we define features as binary indicators for whether the concept is recorded in the past 30 days. Let $\hat{f}_t^{30}$ denote the outcome models learned with these features. To examine our choice of 30-day feature windows, we also assess temporal shift in models $\hat{f}_t^{365}$ built with 365-day feature windows for the top 5 condition outcomes from 2017 to 2020. We perform this assessment at the population level. Let $h_e$ denote a sub-population model that selects the entire population, that is, $h_e\left(x, y\right) = 1 \forall x \in \mathcal{X}, y \in \mathcal{Y}$. For each of the 15 tasks, we compare the AUC difference $\phi_{\mathcal{D}_t}\left(\hat{f}_t^{30}, h_e\right) - \phi_{\mathcal{D}_t}\left(\hat{f}_{t-1}^{30}, h_e\right)$ evaluated by our algorithm for models built with 30-day features against the difference $\phi_{\mathcal{D}_t}\left(\hat{f}_t^{365}, h_e\right) - \phi_{\mathcal{D}_t}\left(\hat{f}_{t-1}^{365}, h_e\right)$ for models with 365-day features. \figureref{fig:feature_window_comparison} shows these differences are highly correlated. Thus, we expect similar conclusions for different choices of feature windows.

\begin{figure}[t]
    \floatconts
    {fig:feature_window_comparison}
    {\vspace{-15px}\caption{Comparison of AUC difference between current and previous model when using 30-day feature window vs 365-day feature window. 3 points per outcome from different years (2018-2020). $y = x$ shown for reference. GERD: gastroesophageal reflux disease.}}
    {\includegraphics[width=0.96\linewidth]{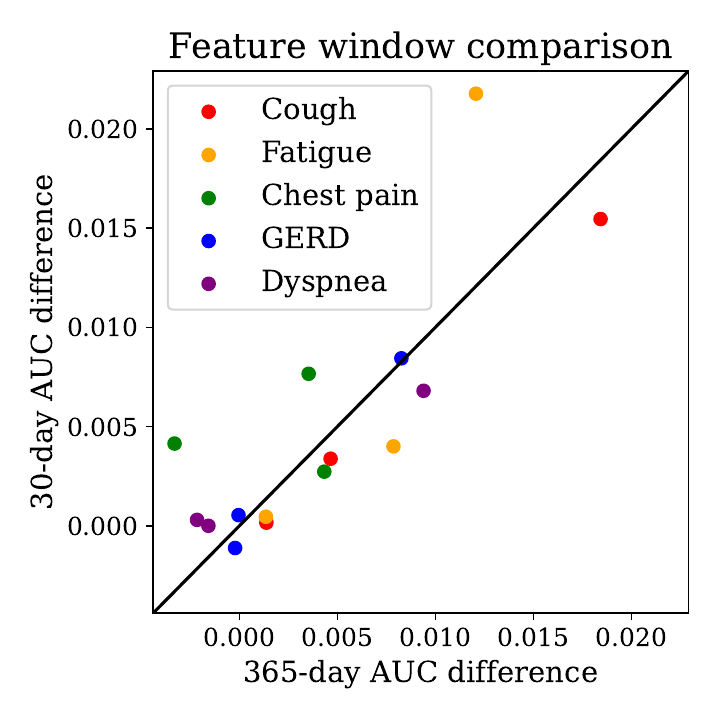}}
    \vspace{-5px}
\end{figure}

\clearpage
\section{Sub-population Discovery Details}

\subsection{Alternative Sub-population Problem Formulations}
\label{app:subpop_problem}

As stated in \sectionref{sec:subpop_step}, there are several ways to formulate the problem of learning a sub-population $h_t$ that satisfies \definitionref{def:temporal_shift} for temporal dataset shift. Many different sub-populations $h_t$ can satisfy the definition. We could formalize the problem as 
\begin{equation}
    h^*_t = \argmax_{h} \mathbb{E}_{\mathcal{D}_t}\left[\phi_{\mathcal{D}_t}\left(f_t, h\right)\right] - \mathbb{E}_{\mathcal{D}_t}\left[\phi_{\mathcal{D}_t}\left(f_{t-1}, h\right)\right]
\end{equation}
A counterexample demonstrates why this formulation is not ideal: If $\phi$ is AUC, this metric would be maximized with only two points in the region. Alternatively, we could formulate the problem as finding a maximal size region where the condition holds:
\begin{multline}
    h^*_t = \argmax_{h} \mathbb{E}_{\mathcal{D}_t}\left[\mathds{1}\left\{h\right\}\right] \\
    \text{s.t. } \mathbb{E}_{\mathcal{D}_t}\left[\phi_{\mathcal{D}_t}\left(f_t, h_t\right)\right] - \mathbb{E}_{\mathcal{D}_t}\left[\phi_{\mathcal{D}_t}\left(f_{t-1}, h_t\right)\right] > 0
\end{multline}
A counterexample for this problem definition is a task where the difference is positive when the metric is evaluated across the entire population. This formulation would return the entire population when we would like a more specific sub-population. 

Another issue that affects both formulations is the metric $\phi$ may not be easy to optimize. The formulation we give in \sectionref{sec:subpop_step} with the loss function allows us to learn a model without running into any of these issues.

\subsection{Sub-population Label Definition in Our Scan}
\label{app:subpop_crossent_calibration}

First, we prove the statement in \sectionref{sec:modeling_choices} that the definition for sub-population labels based on the cross entropy loss for logistic regressions gives rise to calibration-based labels.

\begin{theorem}[Sub-population label definitions]
    \label{thm:subpop_labels_equiv}
    Let $\hat{f}_{t-1}$ and $\hat{f}_t$ denote the fitted outcome models at times $t-1$ and $t$. Let $x_i$ and $y_i$ denote the features and outcome, respectively, for sample $i$ at time $t$. Let $z_i$ denote the label assigned to that sample for fitting the sub-population model. We present two definitions for $z_i$:
    \begin{itemize}
        \item[(i)] Cross entropy definition: We set $\mathcal{L}$ to the cross-entropy loss in the definition for $z_i$ given in \equationref{eq:subpop_label}. The label is the sign of the difference between the cross entropy losses:
        \begin{align}
            \label{eq:subpop_label_cross_entropy}
            z_i = \mathds{1} \Bigl\{ & -y_i \log \hat{f}_{t-1}\left(x_i\right) \nonumber \\
            &- \left(1 - y_i\right) \log \left(1 - \hat{f}_{t-1}\left(x_i\right)\right) \nonumber \\
            &+ y_i \log \hat{f}_t\left(x_i\right) \nonumber \\
            &+ \left(1 - y_i\right) \log \left(1 - \hat{f}_t\left(x_i\right)\right) > 0 \Bigr\}
        \end{align}
        \item[(ii)] Calibration-based definition: Let us define $z'_i$ as whether the prediction from $\hat{f}_t$ is closer to the true label than the prediction from $\hat{f}_{t-1}$: 
        \begin{equation}
            \label{eq:subpop_label_calibration}
            z'_i := \mathds{1} \left\{ \vert y_i - \hat{f}_{t-1}\left(x_i\right) \vert > \vert y_i - \hat{f}_t \left(x_i\right) \vert \right\}
        \end{equation}
    \end{itemize}
     Definition (i) implies definition (ii).
\end{theorem}

\begin{proof}
    First, consider the case where $y_i = 0$:
    \begin{align}
        z_i &= \mathds{1}\left\{- \log \left(1 - f_{t-1}\left(x_i\right)\right) + \log \left(1 - f_t\left(x_i\right)\right) > 0\right\} \\
        &= \mathds{1}\left\{\log \left( \frac{1 - f_t\left(x_i\right)}{1 - f_{t-1}\left(x_i\right)}\right) > 0\right\} \label{eq:logy0}
    \end{align}
    The condition inside the indicator function in \equationref{eq:logy0} holds if
    \begin{equation}
        \frac{1 - f_t\left(x_i\right)}{1 - f_{t-1}\left(x_i\right)} > 1
    \end{equation}
    For the logarithms in \equationref{eq:subpop_label_cross_entropy} to be defined, $0 < f_{t-1}\left(x_i\right) < 1$ and $0 < f_t\left(x_i\right) < 1$ for all $x_i$. Thus,
    \begin{equation}
        1 - f_t\left(x_i\right) > 1 - f_{t-1}\left(x_i\right)
    \end{equation}
    \begin{equation}
        - f_{t-1}\left(x_i\right) < - f_t\left(x_i\right)
    \end{equation}
    Taking the absolute values of negative quantities flips the signs, so
    \begin{equation}
        \lvert - f_{t-1}\left(x_i\right) \rvert > \lvert - f_t\left(x_i\right) \rvert
    \end{equation}
    Since $y_i = 0$,
    \begin{equation}
        \lvert y_i - f_{t-1}\left(x_i\right) \rvert > \lvert y_i - f_t\left(x_i\right) \rvert
    \end{equation}
    Now that we have shown Equations ~\ref{eq:subpop_label_cross_entropy} and \ref{eq:subpop_label_calibration} are equivalent when $y_i = 0$, we will complete the proof by showing that they are also equivalent when $y_i = 1$:
    \begin{align}
        z_i &= \mathds{1}\left\{ - \log f_{t-1}\left(x_i\right) + \log f_t\left(x_i\right) > 0\right\} \\
        &= \mathds{1}\left\{ \log \left(\frac{f_t\left(x_i\right)}{f_{t-1}\left(x_i\right)}\right) > 0\right\} \label{eq:logy1}
    \end{align}
    The condition inside the indicator function in \equationref{eq:logy1} holds if
    \begin{equation}
        \frac{f_t\left(x_i\right)}{f_{t-1}\left(x_i\right)} > 1
    \end{equation}
    \begin{equation}
    f_t\left(x_i\right) > f_{t-1}\left(x_i\right)
    \end{equation}
    \begin{equation}
        1 - f_{t-1}\left(x_i\right) > 1 - f_t\left(x_i\right) 
    \end{equation}
    These quantities are positive, so 
    \begin{equation}
        \lvert 1 - f_{t-1}\left(x_i\right) \rvert > \lvert 1 - f_t\left(x_i\right) \rvert
    \end{equation}
    Plugging in $y_i = 1$ gives the desired
    \begin{equation}
        \lvert y_i - f_{t-1}\left(x_i\right) \rvert > \lvert y_i - f_t\left(x_i\right) \rvert
    \end{equation}
\end{proof}

Second, we verify empirically that the difference in cross entropy losses is significant. We compare the distribution of $z_i$ for two models from different years to the distribution for two models learned on different data folds from the same year. We provide an example with the self-care training outcome in 2016, where our algorithm detected temporal shift in the sub-population. The top plot of \figureref{fig:nll_dist} shows the distribution of differences between the cross entropy loss from the 2016 versus the 2015 model. $z_i = 1$ for all samples to the right of the black line. The bottom plot shows the distribution of differences between the cross entropy loss from models learned on two different data folds in 2016. We expect the differences to be smaller since there is no shift between the two folds. The new fold is constructed by swapping the validation patient set with a third of the training patient set. The same feature set is used in the new data fold. As shown in the bottom plot of \figureref{fig:nll_dist}, this distribution is more narrowly peaked around 0. Thus, the differences used to define the sub-population labels capture more than just noise between different models.

\begin{figure}[t]
    \floatconts
    {fig:nll_dist}
    {\vspace{-15px}\caption{Distribution of negative log likelihood differences. Top: Difference between 2016 and 2015 models for predicting the self-care training procedure outcome on training data from 2016. Samples to the right of the black line are labeled with $z = 1$ for the sub-population identification step. Bottom: Difference between model learned on fold 0 vs fold 1 for predicting self-care training procedure outcome. Both are evaluated on fold 0 training data from 2016.}}
    {\includegraphics[width=0.96\linewidth]{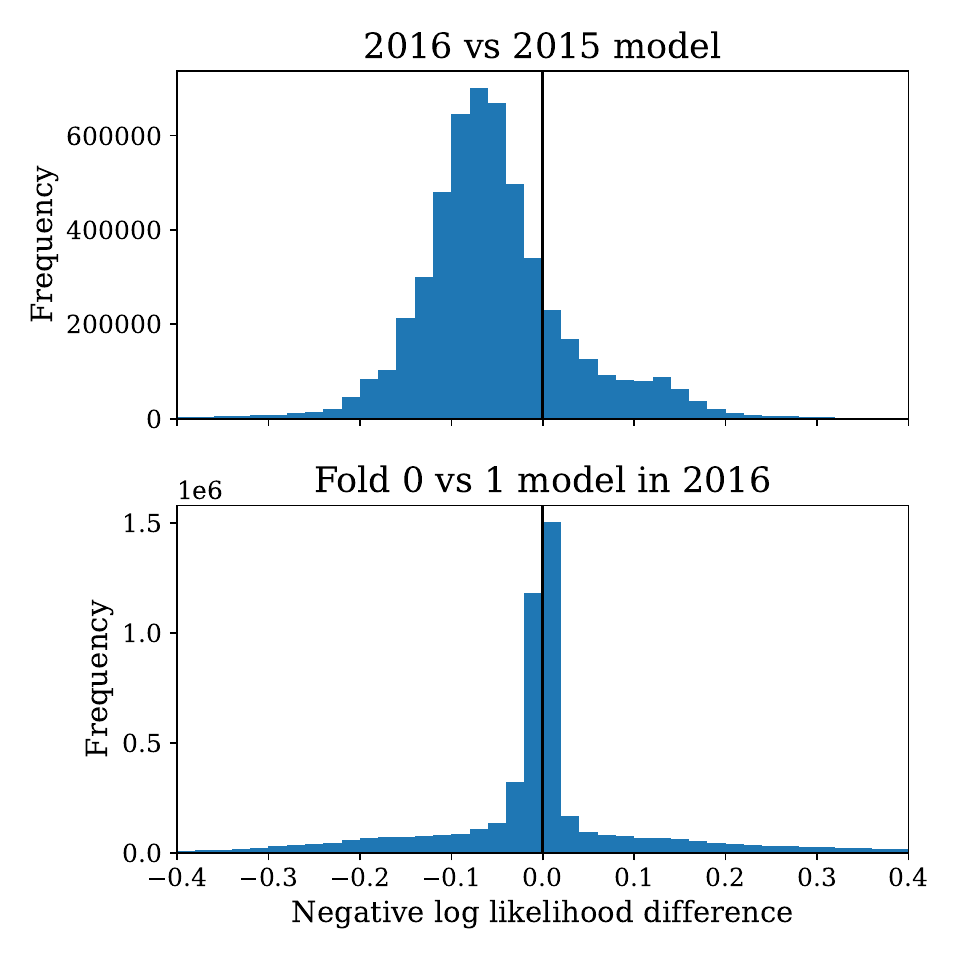}}
    \vspace{-5px}
\end{figure}

\subsection{Sub-population Model}
\label{app:subpop_model}

When learning decision trees for each sub-population, samples are also weighted to address class imbalance. We tune the minimum number of samples per leaf among 10, 25, and 100 and select the model with the best validation AUC. In case the training and validation splits have different prediction distributions, the two splits for learning the outcome models are randomly shuffled and re-split for learning the sub-population models.

For sub-population models, we consider using logistic regressions, decision trees, random forests, or gradient-boosted decision trees. We evaluate the test AUCs also on the top 5 condition outcomes. We also count the number of non-stationary regions with p-values below .05 and AUC differences above .01. We use logistic regressions for the outcome models in this comparison. For sub-population models, logistic regressions are best at predicting where the previous outcome models are not as well calibrated as the current outcome models. However, decision trees are much better at creating non-stationary sub-populations. As shown in \tableref{tab:subpop_model_classes}, we find 14 non-stationary regions with decision trees but none with logistic regressions or random forests. Gradient-boosted trees are more complicated than decision trees but show no additional benefits. Because the goal of learning these models is to identify sub-populations impacted by temporal dataset shift, we choose decision trees for the sub-population models.

\begin{table}[htbp]
\floatconts
  {tab:subpop_model_classes}%
  {\caption{Performance of each model class for sub-population models applied to logistic regression outcome models. Mean and standard deviation in parentheses for test AUC are computed over models from all years for predicting the top 5 condition outcomes. Number of non-stationary regions is the total across these outcomes and years. GB: Gradient-boosted.}}%
  {\begin{tabular}{lcc}
  \toprule
  \bfseries Model class & \bfseries AUC & \bfseries Non-stat. reg. \\
  \midrule
  Logistic reg. & 0.988 (0.004) & 0 \\
  Decision tree & 0.960 (0.013) & 14 \\
  Random forest & 0.932 (0.023) & 0 \\
  GB trees & 0.959 (0.016) & 14 \\
  \bottomrule
  \end{tabular}}
\end{table}

\clearpage
\section{Hypothesis Testing Details}
\label{app:hyp_test_details}
\subsection{Additional Notation for Hypothesis Testing Details}
\label{app:data_struc_notation}
As we note in \appendixref{app:cohort_def}, a dataset $\mathcal{D}_t$ contains multiple samples from the same patient collected at different time points within time frame $t$. When we compute confidence intervals in our algorithms, we account for how patients are independently and identically distributed (iid) while samples for each patient are not. For the permutation tests in our algorithms, we also account for this data structure when determining which observations are exchangeable.

To describe this data structure, we use the following notation in our algorithms: Let $P$ denote the number of patients and $\mathcal{P}$ denote the patient indices $\left[1, \ldots, P\right]$. Let $J_p^{t-1}$ and $J_p^t$ be the number of samples patient $p$ has at times $t-1$ and $t$, respectively. Let $\mathcal{J}_p^{t-1}$ denote the indices $\left[1, \ldots, J_p^{t-1}\right]$ and $\mathcal{J}_p^t$ denote the indices $\left[J_p^{t-1} + 1, \ldots, J_p^{t-1} + J_p^t\right]$. Thus, for $1 \le j \le J_p^{t-1}$, $\left(x_p^j, y_p^j\right)$ denotes the $j$th sample patient $p$ contributes to the dataset $\mathcal{D}_{t-1}$ at time $t-1$. For $J_p^{t-1} + 1 \le j \le J_p^{t-1} + J_p^t$, $\left(x_p^j, y_p^j\right)$ denotes the $\left(j - J_p^{t-1}\right)$th sample patient $p$ contributes to the dataset $\mathcal{D}_t$ at time $t$. Since we are evaluating temporal shift between two time points, we will generally consider two time points $t = 0$ and $t = 1$. Let $\mathcal{J}_p = \mathcal{J}_p^{t-1} \cup \mathcal{J}_p^t$ be the concatenation of the two lists of indices. Let $\mathcal{P}_1$ consist of all patients $p$ such that $y_p^j = 1$ for some $j \in \mathcal{J}_p$. Let $\mathcal{P}_0$ consist of all patients $p$ such that  $y_p^j = 0$ for all $j \in \mathcal{J}_p$.

Since our criteria for hypothesis testing assess models within a sub-population $h_t$, outside the sub-population, and across the entire population, we will use the following notation: Let $h_t^c$ denote a sub-population model that selects the complement of the sub-population selected by $h_t$, that is, $h_t^c\left(x, y\right) = 1 - h_t\left(x, y\right) \forall x \in \mathcal{X}, y \in \mathcal{Y}$. Let $h_e$ denote a sub-population model that selects the entire population, that is, $h_e\left(x, y\right) = 1 \forall x \in \mathcal{X}, y \in \mathcal{Y}$.

\subsection{Criteria for Hypothesis Testing}
\label{app:criteria}
For a hypothesis to be selected for multiple hypothesis testing, the task must pass the following criteria:
    
\textbf{Minimum sample size:} The task must have a minimum number of patients who have the outcome. In \algorithmref{alg:sample_size_check}, we require at least 25 patients have the outcome in the validation split for both years. Since samples from the same patient are more similar with each other, only observing the outcome in a few patients may result in high variance in the estimate of model performance on the general population. Thus, we designed this criterion in terms of the number of patients rather than the number of samples with the outcome. Another motivation for this choice is minimizing repetition among the bootstrap and permutation datasets drawn when computing the confidence intervals and permutation tests described later in this appendix.
   
For sub-population analyses, these minimum frequencies must be satisfied separately inside the region and outside the region. We also require at least 25 patients with at least one sample without the outcome inside the region (and likewise for outside the region). Because the outcomes are rare, this requirement is unnecessary when evaluating the entire population. However, for sub-populations, we are no longer guaranteed sufficient sample size without the outcome. We also check that the sub-population contains between 0.1\% and 75\% of the total number of samples in the validation split in the current year. Sub-populations that are too small or too large may not be clinically meaningful.
   
\textbf{Well-fit outcome models:} We require models $\hat{f}_{t-1}$ and $\hat{f}_t$ be well-fit. Poorly fit models do not capture the true conditional distribution and cannot be used to assess dataset shift. In \algorithmref{alg:model_fit_check}, we check that $\phi_{\mathcal{D}_t}\left(\hat{f}_t, h_e\right) > c_{thr}$ and $\phi_{\mathcal{D}_{t-1}}\left(\hat{f}_{t-1}, h_e\right) > c_{thr}$, where $c_{thr}$ is a defined threshold. When assessing AUC on the validation set, we set the threshold to $c_{thr} = 0.5$.
    
For sub-population analyses, we additionally require the model in year $t$ be well-fit on that sub-population: $\phi_{\mathcal{D}_t}\left(\hat{f}_t, \hat{h}_t\right) > c_{thr}$. For the baseline algorithm, we omit $\phi_{\mathcal{D}_t}\left(\hat{f}_t, h_e\right) > c_{thr}$.
    
\textbf{Current model is significantly better:} As an initial screening for whether a task may be affected by temporal shift, we require the current model $\hat{f}_t$ to have a significantly higher metric value on the validation dataset $\mathcal{D}_t$ than the previous model $\hat{f}_{t-1}$. In \algorithmref{alg:perform_comp}, we first check the current model achieves a higher metric value, that is, $\phi_{\mathcal{D}_t}\left(\hat{f}_t, h_e\right) - \phi_{\mathcal{D}_t}\left(\hat{f}_{t-1}, h_e\right) > 0$. Then, we check this difference is statistically significant by computing a 90\% bootstrap confidence interval for $\mathbb{E}_{\mathcal{D}_t}\left[\phi_{\mathcal{D}_t}\left(\hat{f}_t, h_e\right)\right] - \mathbb{E}_{\mathcal{D}_t}\left[\phi_{\mathcal{D}_t}\left(\hat{f}_{t-1}, h_e\right)\right]$ using \algorithmref{alg:bootstrap_ci_2models_1dataset}.
    
We perform bootstrap by drawing patients rather than samples since bootstrap must be performed on iid items. Following the standard R package for confidence intervals for AUC, we also stratify bootstrap by outcome and perform 2000 bootstrap iterations \citep{robin2011proc}. Putting these factors together, our algorithm proceeds as follows: First, we split patients into two categories: $\mathcal{P}_1$ for those who have the outcome in some sample in the dataset and $\mathcal{P}_0$ for those who never have the outcome. Then, we perform stratified bootstrap at the patient level, that is, draw $\lvert \mathcal{P}_1 \rvert$ patients with replacement from $\mathcal{P}_1$ and $\lvert \mathcal{P}_0 \rvert$ patients with replacement from $\mathcal{P}_0$. When a patient $p$ is sampled, all samples from that patient $\left(x_p^j, y_p^j\right)_{j=1}^{J_p^t}$ are included in the bootstrap dataset $\mathcal{D}_t^{b*}$. We compute the AUC difference $\hat{A}^{b*} = \phi_{\mathcal{D}_t^{b*}}\left(\hat{f}_t, h_e\right) - \phi_{\mathcal{D}_t^{b*}}\left(\hat{f}_{t-1}, h_e\right)$ on each bootstrap dataset $\mathcal{D}_t^{b*}$. Let $\hat{A}_5$ and $\hat{A}_{90}$ denote the 5th and 90th percentiles, respectively, of the bootstrap estimates $\hat{A}^{1*}, \ldots, \hat{A}^{2000*}$. Let $\hat{A}$ denote the AUC difference on the actual dataset. The 90\% confidence interval is defined as $\left[2 \hat{A} - \hat{A}_{90}, 2 \hat{A} - \hat{A}_5\right]$. If this interval is above 0, the task passes this check and may be included for hypothesis testing.

For sub-populations, we instead check the 90\% bootstrap confidence interval for $\mathbb{E}_{\mathcal{D}_t}\left[\phi_{\mathcal{D}_t}\left(\hat{f}_t, \hat{h}_t\right)\right] - \mathbb{E}_{\mathcal{D}_t}\left[\phi_{\mathcal{D}_t}\left(\hat{f}_{t-1}, \hat{h}_t\right)\right]$ within the sub-population is above 0. Additionally, we check the 90\% bootstrap confidence interval for $\mathbb{E}_{\mathcal{D}_t}\left[\phi_{\mathcal{D}_t}\left(\hat{f}_t, \hat{h}_t\right)\right] - \mathbb{E}_{\mathcal{D}_t}\left[\phi_{\mathcal{D}_t}\left(\hat{f}_{t-1}, \hat{h}_t\right)\right]$ outside the sub-population is not above 0. This check excludes sub-populations where temporal shift is also present outside the sub-population.

The baseline check in \algorithmref{alg:baseline_comp} computes a 90\% bootstrap confidence interval for $\mathbb{E}_{\mathcal{D}_{t-1}}\left[\phi_{\mathcal{D}_{t-1}}\left(\hat{f}_{t-1}, h_e\right)\right] - \mathbb{E}_{\mathcal{D}_t}\left[\phi_{\mathcal{D}_t}\left(\hat{f}_{t-1}, h_e\right)\right]$ using the validation datasets to simulate draws of $\mathcal{D}_{t-1}$ and $\mathcal{D}_t$. To account for the same patients contributing different samples to both datasets, the bootstrap procedure for computing confidence intervals in \algorithmref{alg:bootstrap_ci_1model_2datasets} differs from \algorithmref{alg:bootstrap_ci_2models_1dataset} in two ways: First, $\mathcal{P}_1$ contains all patients who have the outcome in either $\mathcal{D}_{t-1}$ or $\mathcal{D}_t$. Second, the two datasets are sampled at once: When a patient $p$ is sampled, $\left(x_p^j, y_p^j\right)_{j=1}^{J_p^{t-1}}$ are included in the bootstrap dataset $\mathcal{D}_{t-1}^{b*}$, and $\left(x_p^j, y_p^j\right)_{j=J_p^{t-1} + 1}^{J_p^{t-1} + J_p^t}$ are included in $\mathcal{D}_t^{b*}$. If the patient has no samples in $\mathcal{D}_t$, no samples are contributed to $\mathcal{D}_t^{b*}$ (likewise for $\mathcal{D}_{t-1}$ and $\mathcal{D}_{t-1}^{b*}$).

\begin{algorithm2e}[tbp]
    \caption{Satisfy sample size check}
    \label{alg:sample_size_check}
    \SetNoFillComment
    \KwIn{Samples $\left\{\left\{\left(x_p^j, y_p^j\right)_{j \in \mathcal{J}_p^{t'}}\right\}_{p \in \mathcal{P}}\right\}_{t' = t-1}^t$ in validation split, sub-population model $h_t$, threshold $n_{thr}$ (default: 25)}
    \KwOut{Boolean: Is sample size sufficient?}
    \eIf{$h_t\left(x, y\right) = 1 \forall x \in \mathcal{X}, y \in \mathcal{Y}$}{
        $S \leftarrow 0$\;
    }{
        $S \leftarrow 1$\;
    }
    \For{$t' \leftarrow t-1 \KwTo t$}{
        \If{$\sum_{p \in \mathcal{P}} \sum_{j \in \mathcal{J}_p^{t'}} \mathds{1}\left\{y_p^j h_t\left(x_p^j, y_p^j\right)\right\} < n_{thr}$}{
            \tcp*[h]{too few patients with outcome}\\
            \KwOut{\textsf{\upshape False}}
        }
        \If{$S = 1$}{
            \If{$\sum_{p \in \mathcal{P}} \sum_{j \in \mathcal{J}_p^{t'}} \mathds{1}\left\{y_p^j \left(1 - h_t\left(x_p^j, y_p^j\right)\right)\right\} < n_{thr}$}{
                \KwOut{\textsf{\upshape False}}
            }
           \If{$\sum_{p \in \mathcal{P}} \sum_{j \in \mathcal{J}_p^{t'}} \mathds{1}\left\{\left(1 - y_p^j\right) h_t\left(x_p^j, y_p^j\right)\right\} < n_{thr}$}{
                \KwOut{\textsf{\upshape False}}
            }
            \If{$\sum_{p \in \mathcal{P}} \sum_{j \in \mathcal{J}_p^{t'}} \mathds{1}\left\{\left(1 - y_p^j\right)  \left(1 - h_t\left(x_p^j, y_p^j\right)\right)\right\} < n_{thr}$}{
                \KwOut{\textsf{\upshape False}}
            } 
        }
    }
    \If{$S = 1$}{
        $m \leftarrow \sum_{p \in \mathcal{P}} J_p^t $\;
        $s \leftarrow \sum_{p \in \mathcal{P}} \sum_{j \in \mathcal{J}_p^t} h_t\left(x_p^j, y_p^j\right)$\;
        \If{$s < .01 * m$}{
            \tcp*[h]{sub-population too small}\\
            \KwOut{\textsf{\upshape False}}
        }
        \If{$s > .75 * m$}{ 
            \KwOut{\textsf{\upshape False}} 
        }
    }
    \KwOut{True}
\end{algorithm2e}

\begin{algorithm2e}[tbp]
    \caption{Satisfy model fit check}
    \label{alg:model_fit_check}
    \KwIn{Samples $\mathcal{D}_{t-1} = \left(X^{t-1}_i, Y^{t-1}_i\right)_{i=1}^{n^{t-1}}$ and $\mathcal{D}_t = \left(X^t_i, Y^t_i\right)_{i=1}^{n^t}$ in validation split, outcome models $f_{t-1}, f_t$, sub-population model $h_t$, $c_{thr}$ threshold for model performance (default: .5 for AUC)}
    \KwOut{Boolean: Are models well-fit?}
    $r_{t-1} \gets \phi_{\mathcal{D}_{t-1}}\left(f_{t-1}, h_e\right)$\;
    \If{$r_{t-1} < c_{thr}$}{
        \KwOut{\textsf{\upshape False}}
    }
    $r_t \leftarrow \phi_{\mathcal{D}_t}\left(f_t, h_e\right)$\; 
    \If{$r_t < c_{thr}$}{
        \KwOut{\textsf{\upshape False}}
    }
    $r_s \leftarrow \phi_{\mathcal{D}_t}\left(f_t, h_t\right)$\; \If{$r_s < c_{thr}$}{
        \KwOut{\textsf{\upshape False}}
    }
    \KwOut{\textsf{\upshape True}}
\end{algorithm2e}

\begin{algorithm2e}[tbp]
    \caption{Compute bootstrap confidence interval for difference in metric value of 2 models on 1 dataset}
    \label{alg:bootstrap_ci_2models_1dataset}
    \SetNoFillComment
    \KwIn{Samples $\mathcal{D}_t = \left\{\left(x_p^j, y_p^j\right)_{j \in \mathcal{J}_p^t}\right\}_{p \in \mathcal{P}}$, outcome models $f_{t-1}, f_t$, sub-population model $h_t$, confidence $1 - \alpha$ (default: $1 - \alpha = .90$), $B$ bootstrap iterations (default: $B = 2000$)}
    \KwOut{Confidence interval}
    $a \leftarrow \phi_{\mathcal{D}_t}\left(f_t, h_t\right) - \phi_{\mathcal{D}_t}\left(f_{t-1}, h_t\right)$\; 
    $A^* \leftarrow \left[\right]$ \tcp*{bootstrap AUC diffs}
    \For{$b \leftarrow 1 \KwTo B$}{
        $\mathcal{D}_t^{b*} \leftarrow \left[\right]$ \tcp*{bootstrap samples}
        \For{$i \leftarrow 1 \KwTo \lvert \mathcal{P}_0 \rvert$}{
            $p \leftarrow$ \textsf{\upshape Draw patient from } $\mathcal{P}_0$\;
            $\mathcal{D}_t^{b*} \leftarrow \mathcal{D}_t^{b*} + \left(x_p^j, y_p^j\right)_{j \in \mathcal{J}_p^t}$\;
        }
        \For{$i \leftarrow 1 \KwTo \lvert \mathcal{P}_1 \rvert$}{
            $p \leftarrow$ \textsf{\upshape Draw patient from } $\mathcal{P}_1$\;
            $\mathcal{D}_t^{b*} \leftarrow \mathcal{D}_t^{b*} + \left(x_p^j, y_p^j\right)_{j \in \mathcal{J}_p^t}$\;
        }
        $A^* \leftarrow A^* + \left[\phi_{\mathcal{D}_t^{b*}}\left(f_t, h_t\right) - \phi_{\mathcal{D}_t^{b*}}\left(f_{t-1}, h_t\right)\right]$\;
    }
    $l \leftarrow 2 a - A^*_{1 - \alpha/2}$ \tcp*{$\left(1 - \alpha/2\right)$ percentile}
    $u \leftarrow 2 a - A^*_{\alpha/2}$\;
    \KwOut{$\left(l , u\right)$}
\end{algorithm2e}

\begin{algorithm2e}[tbp]
    \caption{Satisfy performance comparison}
    \label{alg:perform_comp}
    \SetNoFillComment
    \KwIn{Samples $\mathcal{D}_t = \left\{\left(x_p^j, y_p^j\right)_{j \in \mathcal{J}_p^t}\right\}_{p \in \mathcal{P}}$ in validation split, outcome models $f_{t-1}, f_t$, sub-population model $h_t$, confidence $1 - \alpha$ (default: $1 - \alpha = .90$)}
    \KwOut{Boolean: Is $f_{t-1}$ worse than $f_t$ on $\mathcal{D}_t$ in $h_t$?}
    $a \leftarrow \phi_{\mathcal{D}_t}\left(f_t, h_t\right) - \phi_{\mathcal{D}_t}\left(f_{t-1}, h_t\right)$\; 
    \If{$a \le 0$}{
        \KwOut{\textsf{\upshape False}}
    }
    $\left(l, u\right) \leftarrow $ \textsf{\upshape \algorithmref{alg:bootstrap_ci_2models_1dataset} with inputs $\mathcal{D}_t, f_{t-1}, f_t, h_t, 1 - \alpha$}\;
    \If{$l \le 0$}{
        \KwOut{\textsf{\upshape False}}
    }
    \If{$\exists x \in \mathcal{X}, y \in \mathcal{Y}$ \textsf{\upshape where} $h_t\left(x, y\right) \neq 1$}{
        $\left(l^c, u^c\right) \leftarrow $ \textsf{\upshape \algorithmref{alg:bootstrap_ci_2models_1dataset} with inputs $\mathcal{D}_t, f_{t-1}, f_t, h_t^c, 1 - \alpha$}\;
        \If{$l^c > 0$}{
            \tcp*[h]{opposite outside region}\\
            \KwOut{\textsf{\upshape False}} 
        }
    }
    \KwOut{\textsf{\upshape True}}
\end{algorithm2e}

\begin{algorithm2e}[tbp]
    \caption{Compute bootstrap confidence interval for difference in metric value of 1 model on 2 datasets}
    \label{alg:bootstrap_ci_1model_2datasets}
    \SetNoFillComment
    \KwIn{Samples $\left\{\mathcal{D}_{t'} = \left\{\left(x_p^j, y_p^j\right)_{j \in \mathcal{J}_p^{t'}}\right\}_{p \in \mathcal{P}}\right\}_{t' = t-1}^t$, outcome model $f_{t-1}$, sub-population model $h_t$, confidence $1 - \alpha$ (default: $1 - \alpha = .90$), $B$ bootstrap iterations (default: $B = 2000$)}
    \KwOut{Confidence interval}
    $a \leftarrow \phi\left(f_{t-1}, h_t, \mathcal{D}_{t-1}\right) - \phi\left(f_{t-1}, h_t, \mathcal{D}_t\right)$ \textsf{\upshape on input samples}\; 
    $A^* \leftarrow \left[\right]$ \tcp*{bootstrap AUC diffs}
    \For{$b \leftarrow 1 \KwTo B$}{
        $\mathcal{D}_{t-1}^{b*}, \mathcal{D}_t^{b*} \leftarrow \left[\right], \left[\right]$ \tcp*{bootstrap samples}
        \For{$i \leftarrow 1 \KwTo \lvert \mathcal{P}_0 \rvert$}{
            $p \leftarrow$ \textsf{\upshape Draw patient from } $\mathcal{P}_0$\;
            $\mathcal{D}_{t-1}^{b*} \leftarrow \mathcal{D}_{t-1}^{b*} + \left(x_p^j, y_p^j\right)_{j \in \mathcal{J}_p^{t-1}}$\;
            $\mathcal{D}_t^{b*} \leftarrow \mathcal{D}_t^{b*} + \left(x_p^j, y_p^j\right)_{j \in \mathcal{J}_p^t}$\;
        }
        \For{$i \leftarrow 1 \KwTo \lvert \mathcal{P}_1 \rvert$}{
            $p \leftarrow$ \textsf{\upshape Draw patient from } $\mathcal{P}_1$\;
            $\mathcal{D}_{t-1}^{b*} \leftarrow \mathcal{D}_{t-1}^{b*} + \left(x_p^j, y_p^j\right)_{j \in \mathcal{J}_p^{t-1}}$\;
            $\mathcal{D}_t^{b*} \leftarrow \mathcal{D}_t^{b*} + \left(x_p^j, y_p^j\right)_{j \in \mathcal{J}_p^t}$\;
        }
        $A^* \leftarrow A^* + \left[\phi_{\mathcal{D}_{t-1}^{b*}}\left(f_{t-1}, h_t\right) - \phi_{\mathcal{D}_t^{b*}}\left(f_{t-1}, h_t\right)\right]$\;
    }
    $l \leftarrow 2 a - A^*_{1 - \alpha/2}$ \tcp*{$\left(1 - \alpha/2\right)$ percentile}
    $u \leftarrow 2 a - A^*_{\alpha/2}$\;
    \KwOut{$\left(l , u\right)$}
\end{algorithm2e}

\begin{algorithm2e}[tbp]
    \caption{Satisfy baseline comparison}
    \label{alg:baseline_comp}
    \KwIn{Samples $\left\{\mathcal{D}_{t'} = \left\{\left(x_p^j, y_p^j\right)_{j \in \mathcal{J}_p^{t'}}\right\}_{p \in \mathcal{P}}\right\}_{t' = t-1}^t$ in validation split, outcome model $f_{t-1}$, sub-population model $h_t$, confidence $1 - \alpha$ (default: $1 - \alpha = .90$)}
    \KwOut{Boolean: Is $f_{t-1}$ worse on $\mathcal{D}_t$ than $\mathcal{D}_{t-1}$ in $h_t$?}
    $a \leftarrow \phi_{\mathcal{D}_{t-1}}\left(f_{t-1}, h_t\right) - \phi_{\mathcal{D}_t}\left(f_{t-1}, h_t\right)$\; 
    \If{$a \le 0$}{
        \KwOut{\textsf{\upshape False}}
    }
    $\left(l, u\right) \leftarrow $ \textsf{\upshape \algorithmref{alg:bootstrap_ci_1model_2datasets} with inputs $\mathcal{D}_{t-1}, \mathcal{D}_t, f_{t-1}, h_t, 1 - \alpha$}\;
    \If{$l \le 0$}{
        \KwOut{\textsf{\upshape False}}
    }
    \If{$\exists x \in \mathcal{X}, y \in \mathcal{Y}$ \textsf{\upshape where} $h_t\left(x, y\right) \neq 1$}{
        $\left(l^c, u^c\right) \leftarrow $ \textsf{\upshape \algorithmref{alg:bootstrap_ci_1model_2datasets} with inputs $\mathcal{D}_{t-1}, \mathcal{D}_t, f_{t-1}, h_t^c, 1 - \alpha$}\;
        \If{$l^c > 0$}{
            \tcp*[h]{opposite outside region}\\
            \KwOut{\textsf{\upshape False}} 
        }
    }
    \KwOut{\textsf{\upshape True}}
\end{algorithm2e}

\subsection{Permutation Test for Temporal Shift}
\label{app:permutation_test}

As proposed in \sectionref{sec:hyp_test}, we test the null hypothesis that there is no temporal shift from time $t-1$ to time $t$ against a one-sided alternative hypothesis. Under the alternative hypothesis, the outdated model performs worse than the new model. To state these hypotheses formally,
\begin{align*}
    H_0: \,\, &\text{Data is generated from a distribution } \mathbb{P} \in \mathcal{C}_0 \\
    &\text{such that }\, \phi_{\mathcal{D}_t}\left(\hat{f}_t, h_t\right) - \phi_{\mathcal{D}_t}\left(\hat{f}_{t-1}, h_t\right) = 0 \\
    &\text{where } \, \left(x_p^{j}, y_p^{j}\right)_{j=1}^{12} \overset{iid}{\sim} \mathbb{P} \\
    &\text{for each patient } p \text{ from } 1 \text{ to } P, \\
    &\mathcal{D}_t = \left\{\left(x_p^{j}, y_p^{j}\right)_{j=1}^{12}\right\}_{p = 1}^P
\end{align*}
\begin{align*}
    H_1: \,\, &\text{Data is generated from a distribution } \mathbb{P} \in \mathcal{C}_1 \\
    &\text{such that }\, \phi_{\mathcal{D}_t}\left(\hat{f}_t, h_t\right) - \phi_{\mathcal{D}_t}\left(\hat{f}_{t-1}, h_t\right) > 0 \\
    &\text{where } \, \left(x_p^{j}, y_p^{j}\right)_{j=1}^{12} \overset{iid}{\sim} \mathbb{P} \\
    &\text{for each patient } p \text{ from } 1 \text{ to } P, \\
    &\mathcal{D}_t = \left\{\left(x_p^{j}, y_p^{j}\right)_{j=1}^{12}\right\}_{p = 1}^P
\end{align*}
We construct \algorithmref{alg:permut_2models_1dataset} to run this hypothesis test for any metric given the structure of our data (iid patients with non-iid samples).

As stated in \sectionref{sec:modeling_choices}, we draw from the permutation test in \citet{bandos2005permutation} for significant difference between the AUCs of two models on the same dataset. First, they compute the predictions from the two models. For each model, they compute the rank of each prediction. Because the models may output predictions on different scales, converting to ranks puts the predictions on the same scale. This may not be necessary for other metrics. Then, for each permutation, the ranks or predictions from the two models may be swapped for each sample. Let $\mathcal{R}$ denote the predictions or ranks of a list of samples in a permutation. Let $\mathcal{Y}$ denote the true labels of these samples. With slight abuse of notation, we will denote the metric computed directly from these predictions and labels as $\phi_{\left(\mathcal{R}, \mathcal{Y}\right)}\left(\cdot, h_t\right)$. Note that for the original dataset, if $\mathcal{R}$ comes from $\hat{f}$ and the samples come from $\mathcal{D}$, then $\phi_{\left(\mathcal{R}, \mathcal{Y}\right)}\left(\cdot, h_e\right) = \phi_{\mathcal{D}}(\hat{f}, h_e)$. With $N$ samples, there are $2^N$ permutations. The p-value is the proportion of permutations with a larger AUC difference than what is observed.

A key assumption of permutation tests is the items being swapped are exchangeable. Let $r_0$ and $s_0$ be the ranks of predictions from two models for sample 0. Analogously, $r_1$ and $s_1$ are the ranks for sample 1. Let $g$ be the probability density function for a joint distribution of these 4 ranks. Under the null hypothesis the two models have the same AUC, 
\begin{align}
    &g\left(r_0, s_0, r_1, s_1\right) = g\left(s_0, r_0, r_1, s_1\right) \\
    &= g\left(r_0, s_0, s_1, r_1\right) = g\left(s_0, r_0, s_1, r_1\right)
\end{align}
However, this only holds if samples 0 and 1 are from different patients. If samples 0 and 1 are from the same patient, 
\begin{align}
    &g\left(s_0, r_0, r_1, s_1\right) \neq g\left(r_0, s_0, r_1, s_1\right) \\
    &= g\left(s_0, r_0, s_1, r_1\right) \neq g\left(r_0, s_0, s_1, r_1\right)
\end{align}
Thus, for each patient, we can only choose between swapping all samples between the two models or keeping the original arrangement.

\citet{bandos2005permutation} also propose an approximate version of this test that relies on asymptotic normality of the U-statistic. Without iid samples, we are not guaranteed asymptotic normality. Instead, we use a Monte Carlo approximation of the permutation test with 2000 random permutations. We report the p-value as $\left(1 + m\right)/2001$, where $m$ is the number of permutations with a larger AUC difference than what is observed. \citet{phipson2010permutation} recommend adding 1 in the numerator and denominator since understating the p-value can have serious implications when performing multiple hypothesis testing.

\begin{algorithm2e}[tbp]
    \caption{Permutation test for difference in metric value of 2 models on 1 dataset}
    \label{alg:permut_2models_1dataset}
    \SetNoFillComment
    \KwIn{Samples $\mathcal{D}_t = \left\{\left(x_p^j, y_p^j\right)_{j \in \mathcal{J}_p^t}\right\}_{p \in \mathcal{P}}$, outcome models $f_{t-1}, f_t$, sub-population model $h_t$, $B$ permutations (default: $B = 2000$)}
    \KwOut{P-value}
    $a \leftarrow \phi_{\mathcal{D}_t}\left(f_t, h_t\right) - \phi_{\mathcal{D}_t}\left(f_{t-1}, h_t\right)$\;
    $\mathcal{Y} \leftarrow \left\{\left\{y_p^j\right\}_{j \in \mathcal{J}_p^t}\right\}_{p \in \mathcal{P}}$\;
    $\mathcal{R} \leftarrow \left\{\left\{r_p^j \leftarrow f_{t-1}\left(x_p^j\right)\right\}_{j \in \mathcal{J}_p^t}\right\}_{p \in \mathcal{P}}$\;
    $\mathcal{S} \leftarrow \left\{\left\{s_p^j \leftarrow f_t\left(x_p^j\right)\right\}_{j \in \mathcal{J}_p^t}\right\}_{p \in \mathcal{P}}$\;
    \If{$\phi$ \textsf{\upshape is AUC}}{
        $\mathcal{R} \leftarrow \Bigl\{\left\{r_p^j \leftarrow \right.$ \textsf{\upshape rank}$\left(r_p^j\right)$ \textsf{\upshape among }$\left.\mathcal{R} \right\}_{j \in \mathcal{J}_p^t}\Bigr\}_{p \in \mathcal{P}}$\;
        $\mathcal{S} \leftarrow \Bigl\{\left\{s_p^j \leftarrow \right.$ \textsf{\upshape rank}$\left(s_p^j\right)$ \textsf{\upshape among }$\left.\mathcal{S} \right\}_{j \in \mathcal{J}_p^t}\Bigr\}_{p \in \mathcal{P}}$\;
    }
    $m \leftarrow 0$ \tcp*{\# permut. w/ larger AUC diff}
    \For{$b \leftarrow 1 \KwTo B$}{
        $\mathcal{R}^{b*}, \mathcal{S}^{b*} \leftarrow \left[\right], \left[\right]$\;
        \For{$i \leftarrow 1 \KwTo \lvert \mathcal{P} \rvert$}{
            $w \leftarrow$ \textsf{\upshape Draw from Ber} $\left(0.5\right)$\;
            \eIf{$w = 1$}{
                $\mathcal{R}^{b*} \leftarrow \mathcal{R}^{b*} + \left\{s_p^j\right\}_{j \in \mathcal{J}_p^t}$\;
                $\mathcal{S}^{b*} \leftarrow \mathcal{S}^{b*} + \left\{r_p^j\right\}_{j \in \mathcal{J}_p^t}$\;
            }{
                $\mathcal{R}^{b*} \leftarrow \mathcal{R}^{b*} + \left\{r_p^j\right\}_{j \in \mathcal{J}_p^t}$\;
                $\mathcal{S}^{b*} \leftarrow \mathcal{S}^{b*} + \left\{s_p^j\right\}_{j \in \mathcal{J}_p^t}$\;
            }
        }
        $a^{b*} \leftarrow \phi_{\left(\mathcal{R}^{b*}, \mathcal{Y}\right)}\left(\cdot, h_t\right) - \phi_{\left(\mathcal{S}^{b*}, \mathcal{Y}\right)}\left(\cdot, h_t\right)$\;
        \If{$a^{b*} > a$}{
            $m \leftarrow m + 1$\;
        }
    }
    \KwOut{$\frac{1 + m}{1 + B}$}
\end{algorithm2e}

\subsection{Permutation Test for Baseline}
\label{app:permutation_test_baseline}

For the baseline, the permutation test compares the performance of one model $\hat{f}_{t-1}$ on two datasets $\mathcal{D}_{t-1}$ and $\mathcal{D}_t$. To state the hypotheses formally,
\begin{align*}
    H_0: \,\, &\text{Data is generated from a distribution } \mathbb{P} \in \mathcal{C}_2 \\
    &\text{such that }\, \phi_{\mathcal{D}_{t-1}}\left(\hat{f}_{t-1}, h_t\right) - \phi_{\mathcal{D}_t}\left(\hat{f}_{t-1}, h_t\right) = 0 \\
    &\text{where } \, \left(x_p^{j}, y_p^{j}\right)_{j=1}^{24} \overset{iid}{\sim} \mathbb{P} \\
    &\text{for each patient } p \text{ from } 1 \text{ to } P, \\
    &\mathcal{D}_{t-1} = \left\{\left(x_p^{j}, y_p^{j}\right)_{j=1}^{12}\right\}_{p = 1}^P,\\
    &\mathcal{D}_t = \left\{\left(x_p^{j}, y_p^{j}\right)_{j=13}^{24}\right\}_{p = 1}^P
\end{align*}
\begin{align*}
    H_1: \,\, &\text{Data is generated from a distribution } \mathbb{P} \in \mathcal{C}_3 \\
    &\text{such that }\, \phi_{\mathcal{D}_{t-1}}\left(\hat{f}_{t-1}, h_t\right) - \phi_{\mathcal{D}_t}\left(\hat{f}_{t-1}, h_t\right) > 0 \\
    &\text{where } \, \left(x_p^{j}, y_p^{j}\right)_{j=1}^{24} \overset{iid}{\sim} \mathbb{P} \\
    &\text{for each patient } p \text{ from } 1 \text{ to } P, \\
    &\mathcal{D}_{t-1} = \left\{\left(x_p^{j}, y_p^{j}\right)_{j=1}^{12}\right\}_{p = 1}^P, \\
    &\mathcal{D}_t = \left\{\left(x_p^{j}, y_p^{j}\right)_{j=13}^{24}\right\}_{p = 1}^P
\end{align*}
$\mathcal{C}_2$ is a collection of distributions without temporal shift according to the baseline definition. $\mathcal{C}_3$ is a collection of distributions with temporal shift according to the baseline definition.

This test resembles the standard permutation test for comparing whether two distributions have the same value for a statistic. However, we must again be cautious about exchangeability. This time, for each patient, we can only choose between swapping all samples between the two timepoints or keeping the original arrangement. This permutation test procedure is given in \algorithmref{alg:permut_1model_2datasets}.

\begin{algorithm2e}[tbp]
    \caption{Permutation test for difference in metric value of 2 models on 1 dataset}
    \label{alg:permut_1model_2datasets}
    \SetNoFillComment
    \KwIn{Samples $\left\{\mathcal{D}_{t'} = \left\{\left(x_p^j, y_p^j\right)_{j \in \mathcal{J}_p^{t'}}\right\}_{p \in \mathcal{P}}\right\}_{t' = t-1}^t$, outcome model $f_{t-1}$, sub-population model $h_t$, $B$ permutations (default: $B = 2000$)}
    \KwOut{P-value}
    $a \leftarrow \phi_{\mathcal{D}_{t-1}}\left(f_{t-1}, h_t\right) - \phi_{\mathcal{D}_t}\left(f_{t-1}, h_t\right)$\;
    $m \leftarrow 0$ \tcp*{\# permut. w/ larger AUC diff}
    \For{$b \leftarrow 1 \KwTo B$}{
        $\mathcal{D}_{t-1}^{b*}, \mathcal{D}_t^{b*} \leftarrow \left[\right], \left[\right]$\;
        \For{$i \leftarrow 1 \KwTo \lvert \mathcal{P} \rvert$}{
            $w \leftarrow$ \textsf{\upshape Draw from Ber} $\left(0.5\right)$\;
            \eIf{$w = 1$}{
                $\mathcal{D}_{t-1}^{b*} \leftarrow \mathcal{D}_{t-1}^{b*} + \left(x_p^j, y_p^j\right)_{j \in \mathcal{J}_p^t}$\;
                $\mathcal{D}_t^{b*} \leftarrow \mathcal{D}_t^{b*} + \left(x_p^j, y_p^j\right)_{j \in \mathcal{J}_p^{t-1}}$\;
            }{
                $\mathcal{D}_{t-1}^{b*} \leftarrow \mathcal{D}_{t-1}^{b*} + \left(x_p^j, y_p^j\right)_{j \in \mathcal{J}_p^{t-1}}$\;
                $\mathcal{D}_t^{b*} \leftarrow \mathcal{D}_t^{b*} + \left(x_p^j, y_p^j\right)_{j \in \mathcal{J}_p^t}$\;
            }
        }
        $a^{b*} \leftarrow \phi_{\mathcal{D}_{t-1}^{b*}}\left(f_{t-1}, h_t\right) - \phi_{\mathcal{D}_t^{b*}}\left(f_{t-1}, h_t\right)$\;
        \If{$a^{b*} > a$}{
            $m \leftarrow m + 1$\;
        }
    }
    \KwOut{$\frac{1 + m}{1 + B}$}
\end{algorithm2e}

\subsection{Additional Discussion on Baseline}
\label{app:baseline}

An advantage of our algorithm compared to the baseline is that it isolates the effects of distribution shift. As argued in \citet{koh2021wilds}, the baseline is affected by other factors, such as changes in predictability or noise levels in the current year. Learning a new model on data from the current year captures these other factors.

The baseline may react faster to incoming data if deployed online. Because a new model does not need to be trained on the current year, the baseline does not require collecting as much data from the current year. Both the baseline and our algorithm do not react immediately to temporal shift since they require waiting until the outcome window has passed to evaluate predictions against the true label. In our scan, this period is 3 months.

Because learning the sub-populations requires the current model, we do not test for sub-population level shift with the baseline. There is no straightforward way to learn sub-populations only with the previous model. One may consider learning a model to predict the sub-population among samples from the current year where the previous model is incorrect. However, the previous model may also make this mistake on similar data from the previous year. Thus, this error analysis does not necessarily capture temporal shift.

\clearpage
\section{Guidelines for Large Studies on Health Insurance Claims Data}
\label{app:guidelines}

\subsection{Standardized Data Formats}

\label{sec:standardized_data}
We found that using a database with standard table definitions facilitates the data extraction process. We extract data from OMOP CDM tables \citep{hripcsak2015observational}. The person table provides patient demographic information. The observation period table lists date ranges for when each patient is enrolled. The condition occurrence, procedure occurrence, and drug exposure tables include dates when each concept occurred, and the corresponding concept name can be found in the concept table. The visit occurrence table enumerates when patients utilized the healthcare system and the doctor who provided the service. The doctor's specialty can be found in the provider table. Finally, the measurement table tracks lab orders. When the insurance company has a contract with the center that performed the lab, the measurement values and reference ranges are also provided. To extract these features, we use a package built on the publicly available omop-learn repository from \citet{kodialam2021deep}.

\subsection{Cohort Sequence Definition}
\label{app:cohort_def}

To study temporal dataset shift, we set up a sequence of cohorts in a way that maximizes the available sample size. Models are typically trained with a fixed prediction date: The model ingests data recorded prior to this date and predicts an outcome that occurs in a window following this date. We choose to examine shift from year to year. Thus, cohorts in the sequence have prediction dates from different years. To capture year-round trends, each cohort uses 12 prediction dates, arbitrarily defined as the first of each month. To account for seasonality, indicators for the month are included as features. In our experiments, we obtain a sequence of cohorts from 2015 to 2020.

There are many ways to divide patients among these prediction dates. One option is to randomly assign each patient to one prediction date. This strategy substantially under-utilizes the available sample size. Instead, we extract a sample for every prediction date for which a patient is eligible, that is, the patient is observed for at least 95\% of the previous year and is still observed 3 months after the prediction date or has a death recorded in the 3-month window. In our experiments, the cohort satisfying these criteria consists of over 1.6 million people, each of whom contributes an average of 34 samples. The only caveat is we must be cautious when splitting samples to ensure that no information from the test set is leaked into the training or validation sets. All samples from a particular patient are placed in the same split. To balance the outcome frequency across time in each split, the patients are split stratified by the first month in which the outcome occurs for that patient, with an additional stratification containing all patients for whom the outcome never occurs.

\subsection{Frequent Outcome Selection}
\label{sec:outcome_selection}
To automate selection of a large outcome panel, we start with 100 each of the most frequent conditions, procedures, and abnormal lab measurements in the cohort. Frequency is calculated as the number of patient and prediction date pairs where the outcome occurs in the next 3 months. Upon examining this initial foray, we altered the definitions for each outcome type:
\begin{enumerate}
    \item For conditions, instead of defining each occurrence as an instance of the outcome, we only include initial diagnoses. Whereas repeated procedures or lab measurements tend to indicate recurring risk, predicting documentation of chronic conditions at multiple visits provides little value. Once patients have been diagnosed with the condition, they are excluded from the cohort. To ensure the outcomes are initial diagnoses, we require the patient is observed for at least 95\% of the previous 3 years. This prior observation criteria reduces the cohort size to 872,775 people with an average of 29 samples per person. The cohort is further reduced to various sizes based on the initial diagnosis requirement for each outcome. For instance, the cohort for the most frequent condition contains 817,441 people with an average of 22 samples per person. With two additional years of prior observation required, cohorts are only extracted for 2017 to 2020.
    \item For abnormal labs, examining the most frequent outcomes helped us identify several data cleaning steps. Because labs are particularly messy, measurement data processing warrants a separate discussion in \appendixref{sec:labs}. 
    \item For procedures, many concept codes map to almost the same procedure. Learning predictive models to differentiate between these similar concepts is not sensible. We initially leveraged the concept relationship and ancestor tables. However, while some concepts may be closely related to other concepts that share a common ancestor 2 or 3 levels up the hierarchy, other concepts can only be grouped with some of their siblings. Automating such groupings proved to be infeasible. Instead, we group related concepts into a single procedure outcome based on concept names. Each group includes any concept that contains a user-specified string and excludes any concept containing any string within user-specified exclusion set. We share the criteria we crafted in our repository.
\end{enumerate}
After making these modifications, we re-extracted the top 100 condition and lab outcomes and grouped the top procedure outcomes we found. Our final panel has a total of 242 outcomes.

\subsection{Lab Measurement Cleaning}
\label{sec:labs}

Two sources of noise in the lab measurement data are inconsistent references and zero-imputed missing values. We call for a clinically curated set of reference ranges and more standardized handling of missing lab values to advance future research with lab measurements. In the absence of such a resource, we describe our process for creating cleaner lab features.

Different lab centers provide inconsistent reference ranges. Contributing factors include varying ranges in clinical guidelines and preference for flagging only high risk values versus giving early warnings. To create an outcome definition that is uniform across patients, we unify the reference ranges. For the frequent outcomes in our panel, we looked up gender-specific references from clinical sources and manually created a reference table. We found that some labs can be reported in multiple units. We provide references in different units and check that the references match when units are converted. When the provided unit is not sensible for that lab, the values are dropped. When units are not provided, typically a standard unit can be assumed. However, when multiple standard units are possible, unclear values are dropped. Finally, values that are outside the possible range are also removed. We share the specifications we curate in our repository.

Because there are over 2300 lab concepts, manually entering the reference ranges for all labs is impossible without a curated resource. Thus, we automate the remainder of the reference unification process. For each gender and age range (below 30, 30 to 50, 50 to 70, and above 70), we find the most frequent reference for that demographic. Some demographic groups do not have any reference value that occurs at least 5 times. In these cases, we use the reference for a similar lab concept. Similar concepts are defined as sharing the first 15 characters in the concept name and having an average value that is at most 15 units apart for that demographic group. If a similar concept is not available, we fill in references from the opposite gender and nearby age ranges. For remaining labs, we disregard demographic differences and use a reference value that occurs at least 5 times across the entire population. We apply this logic to create standardized lab reference tables.

A second contributor to messy lab data is the incorrect documentation of missing values as zero instead of null. We discovered this issue from unreasonably high frequencies of abnormally low lab outcomes. To correct these errors, we create a heuristic for identifying which zeros are measured values versus placeholders for missing data. Our heuristic has two criteria for determining whether zero is an infeasible value for that lab: 1. The concept does not have negative measurements. 2. If a reference range is available, the low end is above 5. If a reference is unavailable, the average non-zero value is above 5. For labs where both criteria are satisfied, zero measurements are treated as missing values. We apply this heuristic to create a cleaned measurement table.

We create 10 indicator features for each lab: whether the lab was ordered; whether the value is below, in, or above the reference range; which quartile the value falls in; and whether the value has increased or decreased compared to the previous measurement for that patient. Extracting these 10 derived features from a measurement table with over 290 million rows is a time-consuming part of the automated feature extraction process described in \appendixref{app:feature_extract}. To make this part more efficient, we recommend indexing columns in the standardized reference and measurement tables that are created when cleaning the lab data.

\subsection{Efficient Automated Feature Extraction}
\label{app:feature_extract}

Building an automated and efficient feature extraction pipeline is a key component of working with large datasets. As mentioned in \sectionref{sec:standardized_data}, the first step is to extract patient demographics and the conditions, procedures, drug exposures, and provider specialties for each visit using the package from \citet{kodialam2021deep}. Next, we create binary indicators for whether the feature occurred in the past 30 days. Finally, we include features that occur at least 100 and 300 times in the cohorts for condition and other outcomes, respectively. In our experiment, the number of features varies for each outcome. For example, the set-ups for the most frequent condition and lab outcomes have 15,534 and 16,648 features, respectively.

While building this pipeline, we optimized the process to run in a reasonable amount of time and memory. Our first take-away is to share as much of the feature extraction process across cohorts as possible. Instead of extracting the features separately for each prediction date that a patient is observed, we extract one set of features for all patients on the final prediction date. Then, we create the windowed features separately for each prediction date. These features are shared across all outcomes, so the pipeline only diverges in the final step when additional eligibility criteria for condition outcomes and outcome-stratified data splits are applied.

Our second optimization principle is striking a balance between using disk space and virtual memory. The windowed features are computed by aggregating across a matrix with three dimensions: person, feature, and time. This large matrix is built by creating lists of coordinates for each entry. Appending to lists in Python requires allocating new space in memory each time the allotted space is filled. Because this process is slow and memory intensive for long lists, we write small chunks of these lists to disk repeatedly to keep list sizes small. The chunks can be read in at the end to construct the matrix. By utilizing some more disk space, we can significantly reduce the virtual memory requirement and run time.

\clearpage
\section{Domain Shift Details}

\subsection{Additional Analysis of Features with Domain Shift in 2020}
\label{app:cov_shift}

When fitting the logistic regression for predicting whether a sample is from 2019 or 2020, we apply L1 regularization and use the liblinear solver. We selected $C = 0.01$ as the regularization constant because that logistic regression is within 0.01 of the best validation AUC and has a sparser set of features with non-zero coefficients.

Continuing the analysis in \sectionref{sec:cov_shift_2020}, we examine some of the features that shifted in 2020 further here. While the frequency of 2-view chest X-rays decreased, the frequency of 1-view chest X-rays increased. The decrease is likely associated with the decrease in emergency visits. The increase can be accounted for by the use of portable chest X-rays to examine lung abnormalities among COVID patients \citep{jacobi2020portable}.

We also explore the increased prescription rate of statins. First, statin therapy is recommended for patients with coronary artery calcium (CAC) over 100. Health systems started providing \$50 CAC testing in June 2019 after the 2018 guidelines from the American College of Cardiology and American Heart Association advocated for their use \citep{shetty2021effect}. More prevalent testing led to increased statin prescription. Second, another study by \citet{mizuno2021statin} found that statin prescription rates were higher in telemedicine visits than in-person visits in April and May 2020. Both of these factors may have contributed to increased statin prescription despite fewer hyperlipidemia claims.

Many of the shifts in 2020 are complicated by other types of shift and cannot be addressed by the approaches we propose in \sectionref{sec:cov_shift_2020}. For instance, many of the outcomes exhibit label shift since their frequencies also changed in 2020. Conditional shift may also affect some outcomes in 2020. For instance, missed doctor visits may affect outcomes due to poor disease management.

\begin{figure}[h]
    \floatconts
    {fig:domain_shift_proof}
    {\vspace{-15px}\caption{Domain shift affects $X$ but not $X'$.}}
    {\includegraphics[width=0.5\linewidth]{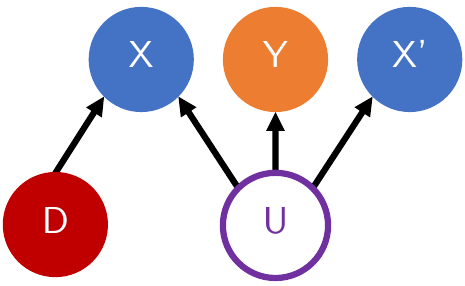}}
    \vspace{-20px}
\end{figure}

\subsection{Analysis of Desired Feature Set for Handling Domain Shift}
\label{app:domain_shift_theory}

\begin{theorem}[Domain Shift: Desired Features]
\label{thm:domain_shift}
Consider the causal graph in \figureref{fig:domain_shift_proof} where $X$ and $X'$ are two views of the unobserved variable $U$. Let $\mathcal{X}$ and $\mathcal{X}'$ be the feature spaces for $X$ and $X'$, respectively. Let $\mathcal{U}$ be the latent space for $U$. If the shift between distributions $\mathbb{P}_{t-1}$ and $\mathbb{P}_t$ satisfies the following conditions:
\begin{enumerate}
    \item No latent shift: $\mathbb{P}_t\left(U\right) = \mathbb{P}_{t-1}\left(U\right)$
    \item No label shift: $\mathbb{P}_t\left(Y \vert U\right) = \mathbb{P}_{t-1}\left(Y \vert U\right)$
    \item Domain shift affects $X$: $\exists x \in \mathcal{X}, u \in \mathcal{U}$ such that $\mathbb{P}_t\left(X = x \vert U = u\right) \neq \mathbb{P}_{t-1}\left(X = x \vert U = u\right)$
    \item Domain shift does not affect $X'$: $\forall x' \in \mathcal{X}', u \in \mathcal{U}$, we have $\mathbb{P}_t\left(X' = x' \vert U = u\right) = \mathbb{P}_{t-1}\left(X' = x' \vert U = u\right)$
\end{enumerate}
then
\begin{equation}
    \label{eq:thm_ds_eq}
    \mathbb{P}_t\left(Y \vert X'\right) = \mathbb{P}_{t-1}\left(Y \vert X'\right)
\end{equation}
However, we are not guaranteed
\begin{equation}
    \label{eq:thm_ds_neq}
    \mathbb{P}_t\left(Y \vert X\right) = \mathbb{P}_{t-1}\left(Y \vert X\right)
\end{equation}
\end{theorem}
\begin{proof}
    We expand the conditional probability by iterating over the latent space $\mathcal{U}$: 
    \begin{align}
        \label{eq:thm_ds_pf_1}
        &\mathbb{P}_t\left(Y \vert X'\right) = \frac{\mathbb{P}_t\left(Y, X'\right)}{\mathbb{P}_t\left(X'\right)} \nonumber \\
        &= \frac{\sum_{u \in \mathcal{U}} \mathbb{P}_t\left(Y, X' \vert U = u\right) \mathbb{P}_t\left(U = u\right)}{\sum_{u \in \mathcal{U}} \mathbb{P}_t\left(X' \vert U = u\right) \mathbb{P}_t\left(U = u\right)} \nonumber \\
        &= \frac{\sum_{u \in \mathcal{U}} \mathbb{P}_t\left(Y \vert U = u\right) \mathbb{P}_t\left(X' \vert U = u\right) \mathbb{P}_t\left(U = u\right)}{\sum_{u \in \mathcal{U}} \mathbb{P}_t\left(X' \vert U = u\right) \mathbb{P}_t\left(U = u\right)}
    \end{align}
    The last line follows from $Y \perp X' \vert U$. Note that \equationref{eq:thm_ds_pf_1} still holds if $t$ is replaced with $t-1$ or $X'$ with $X$. We apply the no latent shift, no label shift, and $X'$ unaffected by domain shift assumptions to each term in \equationref{eq:thm_ds_pf_1} to obtain \equationref{eq:thm_ds_eq}. Because domain shift affects $X$, there exists some $u \in \mathcal{U}$ for which $\mathbb{P}_t\left(X \vert U = u\right) \neq \mathbb{P}_{t-1}\left(X \vert U = u\right)$. Thus, \equationref{eq:thm_ds_neq} may not hold.    
\end{proof}

\subsection{Examples of Challenges to Addressing Domain Shift}
\label{app:empirical_domain_shift}

Our first approach is to select more robust features based on clinical intuition rather than by applying machine learning methods. We test two versions of this approach. In the first version, we build indicators for whether features occur in the past 365 days instead of the past 30 days. We expect these features to be less affected by domain shift due to the COVID-19 pandemic. In the second version, we only include demographics, prediction month, and drugs in the past 30 days in the feature set. As observed in \figureref{fig:covariate_shift_2020}, prescriptions continued in 2020, so drug-related features are less affected. We learn new models using these feature sets in both 2019 and 2020 and assess whether the new pair of models still exhibit temporal shift.

Our second approach is to impute missing features similar to the baseline in \citet{zhou2022domain}. We focus on missingness because that seems to be the mechanism underlying many of the shifts we observed. First, we identify features that are important in the 2019 model and have decreased frequency in the 2020 model. Then, we train logistic regressions to predict these features from indicators for whether features occurred in the past 365 days. Finally, when making predictions, if a sample does not have a feature recorded, we impute the feature with a prediction from the logistic regression. To assess whether this approach can address temporal shift without using data from 2020, the feature imputation models are fit using 2019 data, and predictions with the imputed features are made using the 2019 model. Because the 2020 model may also perform better with imputed features, for comparison, we also fit feature imputation models using 2020 data before applying the 2020 model.

We test both approaches on a model predicting headaches. Label shift is not an issue for this outcome because its frequency stays constant around 0.51\% in both years. We apply each approach to obtain separate 2019 and 2020 models. Then, we assess temporal shift by comparing the pair on test data from 2020. \tableref{tab:domain_shift} shows neither approach is able to address temporal shift. These results suggest the robust features and imputed features are still affected by domain shift.

We also include importance reweighting as a baseline because domain shift and covariate shift have similar observed shifts in $\mathbb{P}\left(X\right)$. Importance reweighting is a method for adapting to covariate shift that assumes access to labeled data at $t = 0$ and unlabeled data at $t = 1$ \citep{quinonero2008dataset}. This method places higher importance on samples at $t = 0$ that more closely resemble samples at $t = 1$:
\begin{equation}
    \hat{f} = \argmin_f \frac{1}{n_0} \sum_{i=1}^{n_0} \frac{n_0 \mathbb{P}\left(t = 1 \vert x_i^0\right)}{n_1 \left(1 - \mathbb{P}\left(t = 1 \vert x_i^0\right)\right)} \mathcal{L}\left(f\left(x_i^0\right), y_i^0\right)
\end{equation}
We estimate $\mathbb{P}\left(t = 1 \vert x_i^0\right)$ by fitting a logistic regression to predict which year each sample is from. This model is fit using the training and validation samples from both years. We add L2 regularization using scikit-learn \citep{pedregosa2011scikit}. 95\% of weights fall between 0.13 and 1.52. Weights are clipped between 0.01 and 10 to reduce variance. Importance reweighting is also incapable of addressing domain shift, as evidenced by the low AUC of 0.6108. 

\begin{table}[h]
\floatconts
  {tab:domain_shift}%
  {\caption{Performance of each approach for addressing domain shift for the non-stationary headache outcome in 2020. AUC is evaluated on test set. Imp rwt: importance reweighting}\vspace{-10px}}%
  {\begin{tabular}{lcc}
  \toprule
  \bfseries Method & \bfseries 2019 AUC & \bfseries 2020 AUC \\
  \midrule
  Original 2019 & 0.6377 & 0.6160 \\
  Original 2020 & N/A & 0.6335 \\
  \midrule
  365-day 2019 & 0.6844 & 0.6777 \\
  365-day 2020 & N/A & 0.6915 \\
  \midrule
  Drugs 2019 & 0.5977 & 0.5812 \\
  Drugs 2020 & N/A & 0.6064 \\
  \midrule
  Impute 2019 & 0.6672 & 0.6491 \\
  Impute 2020 & N/A & 0.6677 \\
  \midrule
  Imp rwt 2019 & N/A & 0.6108 \\
  \bottomrule
  \end{tabular}}
\end{table}

\clearpage
\section{Conditional Shift Details}

\subsection{Feature Selection for statsmodels}
\label{app:cond_shift_feat_select}

When examining conditional shift, we use statsmodels \citep{seabold2010statsmodels} to construct confidence intervals for the coefficients since the logistic regressions we originally learned with scikit-learn do not have confidence intervals. Statsmodels does not take sparse feature matrices as input, and our data is too large to fit in a non-sparse matrix. Thus, we perform the following feature selection steps using training data from the two adjacent years being evaluated. Our main goal is to select relevant features and reduce multi-collinearity. When multi-collinearity is present, model parameters may not be identifiable, and statsmodels is unable to compute confidence intervals for the coefficients.
\begin{enumerate}
    \item Age, race, ethnicity, and prediction month are kept. The indicator for prediction in December is removed because the 12 prediction month features are perfectly collinear.
    \item Because features for each lab are closely related to each other, only the indicators for whether each lab was ordered are considered as potential candidates.
    \item For each feature in each year, we compute a chi-squared statistic to test whether the feature is independent from the outcome. Features with the top 100 chi-squared statistics in each year are kept.
    \item Features with frequency below 100 in the training data in either year are removed.
    \item For each pair of selected features, if the feature frequencies are within 100 of each other and the Pearson correlation coefficient is above .95 in either year, the feature with lower frequency is removed.
\end{enumerate}
After performing this feature selection process, the statsmodels logistic regressions we build in the conditional shift analyses for inpatient consultations and nursing care use 144 and 139 features, respectively.

\subsection{Additional Analysis of Conditional Shift Examples}
\label{app:cond_shift_ex_stats}

We provide details on the two conditional shift examples to demonstrate that sample size is large in the affected populations and covariate shift does not affect the features identified to have conditional shift. For the inpatient consultation outcome with conditional shift in 2019, the outcome is defined by 5 CPT-4 codes for inpatient consultation. The two features identified are congestive heart failure and atherosclerosis of coronary artery without angina pectoris. 24,857 and 27,077 samples have congestive heart failure in 2018 and 2019, respectively. These correspond to .29\% and .32\% of the population, so congestive heart failure is not affected by covariate shift. 98,156 and 98,962 samples have atherosclerosis of coronary artery without angina pectoris in 2018 and 2019, respectively. The frequency in both years is 1.2\%, so this feature is also unaffected by covariate shift.

We further examine how changes in Medicare reimbursement policies led to conditional shift in the inpatient consultation outcome. Since Medicare stopped reimbursing inpatient consultations in March 2018, inpatient consultation codes were used much less frequently for Medicare patients. We define Medicare patients as those who are enrolled in a Medicare Health Maintenance Organization plan or a Medicare Preferred Provider Organization plan on the prediction date. We define non-Medicare patients as those who are enrolled in plans unrelated to Medicare, such as Blue Cross plans, on the prediction date. 13\% of patients are in the Medicare cohort, while 85\% of patients are in the non-Medicare cohort. Within the Medicare cohort, the inpatient consultation outcome frequency dropped from 1.3\% in 2018 to 0.4\% in 2019. Within the non-Medicare cohort, the inpatient consultation outcome frequency stayed constant at 0.4\% in 2018 and 2019. Thus, the temporal shift can be explained by the change in Medicare policy.

Because Medicare status is not included as a feature, we hypothesize that congestive heart failure and atherosclerosis are predictive of Medicare status. The data supports this hypothesis since these two conditions are much more prevalent among Medicare patients: Within the Medicare cohort, congestive heart failure frequency is 1.8\% in 2018 and 1.9\% in 2019. Within the non-Medicare cohort, congestive heart failure frequency is below 0.1\% in both years. Similarly, for atherosclerosis in both years, the frequency is 5.7\% in the Medicare cohort and 0.5\% in the non-Medicare cohort. Thus, the conditional shift detected by our algorithm for the inpatient consultation outcome in 2019 can be explained by changes in clinical practice.

For the nursing care outcome with conditional shift in 2016, the outcome is defined by 10 CPT-4 or HCPCS codes related to nursing services, initial nursing facility care, nursing care at home, or home blood transfusions performed by a nurse. Nursing facility discharge in the past 30 days is defined by two CPT-4 codes for nursing facility discharge day management. The coefficients for both features show significant sign changes. Combining the two features, 3,257 and 3,456 samples have nursing facility discharge in the past 30 days in 2015 and 2016, respectively. Recent discharge from a nursing facility shows no covariate shift as the frequency in both years is .04\%. Changes in utilization of nursing care services may be associated with additional penalties on hospitals for re-admission. To prevent re-admission, hospitals may have been more inclined to discharge more severe patients to skilled nursing facilities.

\subsection{Sufficient Conditions for Addressing Conditional Shift via Re-calibration}
\label{app:cond_shift_recal_proof}

\begin{theorem}[Conditional Shift: Re-calibration]
    \label{thm:cond_shift}
    Let the feature set $X$ be partitioned into two sets $B$ and $C$. If conditional shift between distributions $\mathbb{P}_{t-1}$ and $\mathbb{P}_t$ satisfies the following conditions for some $B$ and $C$:
    \begin{enumerate}
        \item No label shift: $\mathbb{P}_t\left(Y\right) = \mathbb{P}_{t-1}\left(Y\right)$
        \item No covariate shift: $\mathbb{P}_t\left(X\right) = \mathbb{P}_{t-1}\left(X\right)$
        \item No conditional shift in partition $B$: $\mathbb{P}_t\left(Y \vert B\right) = \mathbb{P}_{t-1}\left(Y \vert B\right)$
        \item The partitions are conditionally independent given the outcome: $B \perp C \vert Y$
    \end{enumerate}
    then 
    \begin{equation}
        \mathbb{P}_t\left(Y \vert X\right) = \mathbb{P}_{t-1}\left(Y \vert X\right) \frac{\mathbb{P}_t\left(Y \vert C\right)}{\mathbb{P}_{t-1}\left(Y \vert C\right)}
    \end{equation}
\end{theorem}

\begin{proof}
    First, we apply the partition definition:
    \begin{equation}
        \label{eq:cond_factor}
        \frac{\mathbb{P}_t\left(Y \vert X\right)}{\mathbb{P}_{t-1}\left(Y \vert X\right)} 
        = \frac{\mathbb{P}_t\left(Y, B, C\right)}{\mathbb{P}_{t-1}\left(Y, B, C\right)} \frac{\mathbb{P}_{t-1}\left(X\right)}{ \mathbb{P}_t\left(X\right)}
    \end{equation}
    The second fraction on the right is 1 due to the no covariate shift condition. Now, we apply the conditional independence condition:
    \begin{equation}
        \label{eq:joint_factor}
        \frac{\mathbb{P}_t\left(Y, B, C\right)}{\mathbb{P}_{t-1}\left(Y, B, C\right)}
        = \frac{\mathbb{P}_t\left(C \vert Y\right)}{\mathbb{P}_{t-1}\left(C \vert Y\right)} \frac{\mathbb{P}_t\left(B \vert Y\right)}{\mathbb{P}_{t-1}\left(B \vert Y\right)} \frac{\mathbb{P}_t\left(Y\right)}{\mathbb{P}_{t-1}\left(Y\right)}
    \end{equation}
    The third fraction on the right is 1 due to the no label shift condition. For any partition $A$ of $X$, 
    \begin{equation}
        \label{eq:bayes}
        \frac{\mathbb{P}_t\left(A \vert Y\right)}{\mathbb{P}_{t-1}\left(A \vert Y\right)} = \frac{\mathbb{P}_t\left(Y \vert A\right)}{\mathbb{P}_{t-1}\left(Y \vert A\right)} \frac{\mathbb{P}_t\left(A\right)}{\mathbb{P}_{t-1}\left(A\right)} \frac{\mathbb{P}_{t-1}\left(Y\right)}{\mathbb{P}_t\left(Y\right)}
    \end{equation}
    Since the joint covariate distribution has no shift, the marginal distribution of any partition also has no shift. Thus, the second fraction on the right is 1. The third fraction on the right is also 1 due to the no label shift condition. For partition $B$, the first fraction on the right is 1 due to the no conditional shift in partition $B$ condition. Thus, the second fraction on the right in \equationref{eq:joint_factor} is 1. Applying \equationref{eq:bayes} to the first fraction on the right in \equationref{eq:joint_factor} and combining with \equationref{eq:cond_factor}, we get the desired
    \begin{equation}
        \frac{\mathbb{P}_t\left(Y \vert X\right)}{\mathbb{P}_{t-1}\left(Y \vert X\right)} 
        = \frac{\mathbb{P}_t\left(Y \vert C\right)}{\mathbb{P}_{t-1}\left(Y \vert C\right)}
    \end{equation}
\end{proof}

\theoremref{thm:cond_shift} implies conditional shift can be addressed by multiplying the predicted probability by $\frac{\mathbb{P}_t\left(Y \vert C\right)}{\mathbb{P}_{t-1}\left(Y \vert C\right)}$. In the extreme case where $C = \emptyset$, there is no distribution shift, and the ratio is defined as $1$. At the other extreme where $C = X$, the theorem reduces to learning a new model for $\mathbb{P}_t\left(Y \vert X\right)$.

\subsection{Examples of Challenges to Addressing Conditional Shift}
\label{app:cond_shift_challenge_ex}

We discussed two methods for addressing conditional shift in \sectionref{sec:cond_shift}: re-calibrating predictions and learning models using a subset of features. Neither method is able to solve the conditional shift observed in the inpatient consultation and nursing care examples we give. When re-calibrating predictions for the inpatient consultation outcome, we apply multiplicative weight 0.23 to samples with the congestive heart failure feature and .33 to samples with the atherosclerosis feature. For the nursing care outcome, we apply multiplicative weight 3.43 to samples with recent discharge from nursing care recorded under either of the two corresponding concepts. Any predicted probabilities above 1 are clipped to 1. \tableref{tab:cond_shift_auc} shows re-calibrating does not improve the AUC. Because there may be complex interactions between other features and the outcome, these features likely do not satisfy the conditions in \theoremref{thm:cond_shift} required for $C$. 

When selecting features for the second method, we use the confidence intervals computed by statsmodels for each coefficient. Only features with significantly positive coefficients in both years or significantly negative coefficients in both years are included when learning a model with a subset of feature. In particular, features with significant sign changes are excluded. Features with coefficients that are not significantly non-zero in either year are also omitted. 42 and 67 features are selected for the inpatient consultation and nursing care outcomes, respectively. \tableref{tab:cond_shift_auc} shows a significant AUC drop from using these small feature subsets. Omitted features, including the ones we identified with conditional shift, are predictive of the outcome.

\begin{table}[htbp]
\floatconts
  {tab:cond_shift_auc}%
  {\caption{Performance of each approach for addressing conditional shift for the non-stationary inpatient consultation outcome in $t = 2019$ and nursing care outcome in $t = 2016$. AUC is evaluated on the test set at time $t$.}}%
  {\begin{tabular}{lcc}
  \toprule
  \bfseries Model & \bfseries Inpatient & \bfseries Nursing \\
  \midrule
  Original $t$ & 0.7604 & 0.9435 \\
  Original $t - 1$ & 0.7410 & 0.9314 \\
  Re-calibrated $t - 1$ & 0.7123 & 0.9314 \\
  Feature subset $t - 1$ & 0.7045 & 0.8863 \\
  \bottomrule
  \end{tabular}}
\end{table}

\clearpage
\section{Additional Related Work}
\label{app:related_work}

A vast body of literature exists for addressing distribution shift. To name a few, importance reweighting targets covariate shift by assigning higher weights to training samples that are more likely to occur under the new distribution \citep{quinonero2008dataset}. Invariant risk minimization attempts to learn relations that hold across multiple domains \citep{arjovsky2019invariant}. Adaptive risk minimization takes a meta-learning approach to adapt to new distributions at test time \citep{zhang2020adaptive}. Test-time adaptation methods have also been developed to handle unlabeled target domain data \citep{wang2020tent}. However, recent works have demonstrated some of these methods may harm model performance \citep{wiles2021fine,rosenfeld2020risks,kamath2021does}. Because reliable performance is essential for clinical models, we recommend caution when deciding between applying one of these methods and training a new model to address the shifts we observed in healthcare settings.

Some methods have also been developed specifically to mitigate distribution shift in sub-populations. Group distributionally robust optimization minimizes the worst-case loss across pre-defined subgroups \citep{sagawa2019distributionally}. \citet{subbaswamy2021evaluating} identify potential worst-case shifts in sub-populations defined by some mutable features. \citet{song2015dataset} detect subgroups within the training data that are considered different from the general distribution. These works either examine subgroups in training data or define hypothetical subgroups. Our method discovers sub-populations in a new dataset that are affected by distribution shift. We note that sub-populations that are affected by distribution shift are different from the sub-populations studied in a type of distribution shift called sub-population shifts. \citet{santurkar2020breeds} construct sub-population shifts by including different sub-populations in the source and target domains. We are seeking sub-populations that exist in both domains. 

Beyond addressing distribution shift in models, prior works have also focused on identifying samples that fall outside the training distribution. Anomaly detection typically seeks to find examples that fall outside the expected trend in a data stream \citep{chandola2009anomaly}. Anomalies are only defined by the input and are not associated with a label. A common application in healthcare is online detection of anomalies in wearable sensors \citep{salem2014online}. Our work more closely resembles out-of-distribution detection, where they predict whether a test sample is from a different distribution than the training data \citep{hendrycks2016baseline}. However, out-of-distribution benchmarks typically have labels for which samples come from a different distribution than the training dataset and can assess whether those samples have been identified correctly \citep{yang2022openood}. We do not have access to such indicators. 

\citet{lipton2018detecting} use black box predictors to detect distribution shift without such indicators. Their work is focused on detecting label shift and proposes testing for equality of distributions: $H_0: \mathbb{P}_t\left(y\right) = \mathbb{P}_{t-1}\left(y\right)$. Under their assumption that $\mathbb{P}_t\left(x \vert y\right) = \mathbb{P}_{t-1}\left(x \vert y\right)$, they propose testing $H_0: \mathbb{P}_t\left(\hat{y}\right) = \mathbb{P}_{t-1}\left(\hat{y}\right)$ to assess domain shift. Our test is more general since it can detect label, domain, and conditional shift. Furthermore, they use existing tests for equality of distributions. As discussed in Section~\ref{sec:definition}, using our definition of temporal shift allows the test to be more sensitive to clinical impact rather than some other aspect of the distributions.

We also note that our definition of a non-stationary task differs from the definition of non-stationarity in the context of time series. \citet{nason2006stationary} define a time series as stationary if the joint distribution of $X_{t_1}, \ldots, X_{t_l}$ is the same as the joint distribution of $X_{t_1 + \tau}, \ldots, X_{t_l + \tau}$ for all $l$ and $\tau$. They also provide a less strict definition: If the mean and variance do not depend on t and the autocovariance between $X_t$ and $X_{t+\tau}$ only depends on $\tau$, then the series is stationary. The goal of non-stationarity detection in time series is to identify outliers or change points \citep{yamanishi2002unifying}. \citet{sankararaman2022fitness} detect anomalies in non-stationary data streams by fine tuning their anomaly detection model on incoming samples. Applications of detecting non-stationarity in time series include precipitation levels \citep{westra2011detection}, EEG signal \citep{cao2011application}, and traffic volume \citep{vlahogianni2006statistical}.

Finally, we note that insurance claims data have been used to build models for predicting specific outcomes. To give a few examples, \citet{razavian2015population} predict adverse outcomes due to diabetes complications using insurance claims data. \citet{krishnamurthy2021machine} and \citet{segal2020machine} forecast kidney disease, while \citet{kodialam2021deep} predict end-of-life, surgery, and likelihood of hospitalization.

\clearpage
\section{Large-Scale Scan Results}
\label{app:scan_results}

Tables~\ref{tab:condition_results}-\ref{tab:lab_results} list the non-stationary outcomes discovered by our algorithms. To further examine the clinical significance of these shifts, we plot the distribution of the AUC differences $\phi_{\mathcal{D}_t}(\hat{f}_t, h_e) - \phi_{\mathcal{D}_t}(\hat{f}_{t-1}, h_e)$ across all these tasks in the top plot of \figureref{fig:auc_diff_dist}. We also plot the distribution of $\phi_{\mathcal{D}_t}(\hat{f}_t, \hat{h}_t) - \phi_{\mathcal{D}_t}(\hat{f}_{t-1}, \hat{h}_t)$ across all the non-stationary sub-populations in the bottom plot. The AUC differences are much larger within the discovered sub-populations. This means the clinical impact of using an outdated model may be quite large for some patients.

\begin{figure}[h]
    \floatconts
    {fig:auc_diff_dist}
    {\vspace{-15px}\caption{Distribution of difference between AUC of current model and previous model evaluated on current dataset. Top: Non-stationary tasks within the entire population. Bottom: Non-stationary tasks within discovered sub-populations.}}
    {\includegraphics[width=0.96\linewidth]{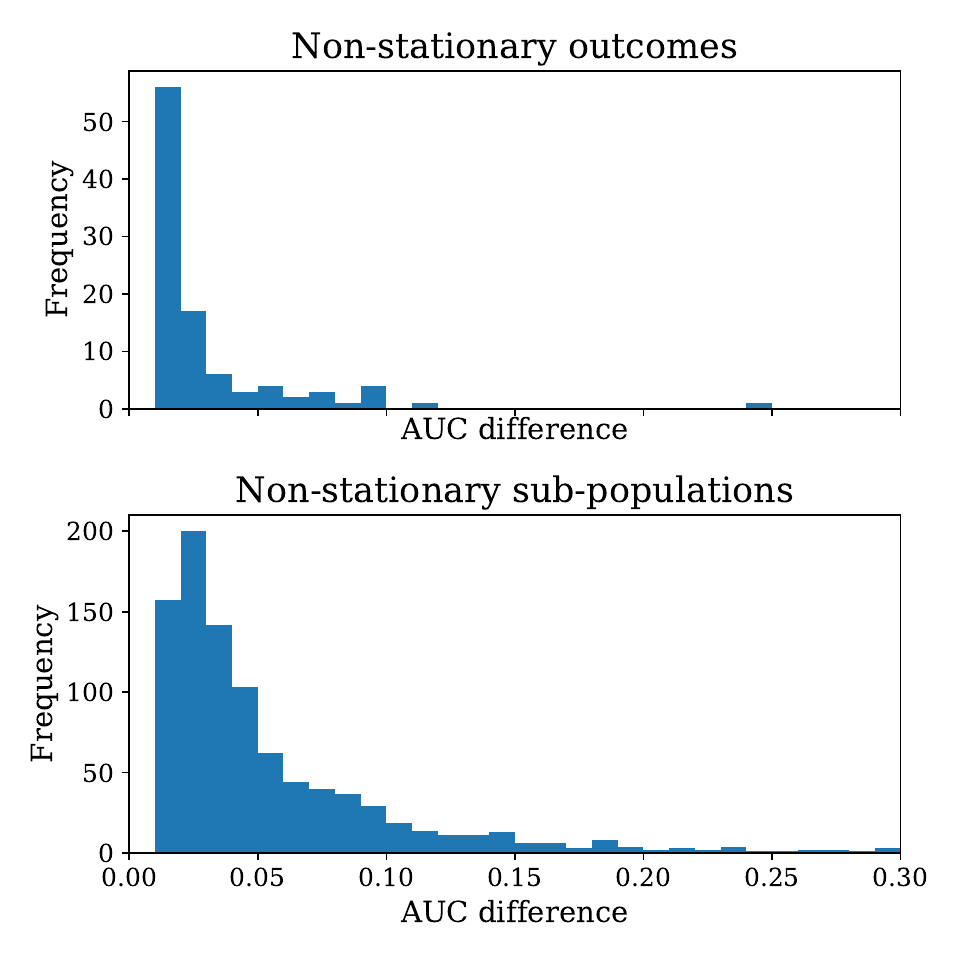}}
    \vspace{-5px}
\end{figure}

\begin{table*}[htbp]
\floatconts
  {tab:condition_results}%
  {\caption{Non-stationary condition outcomes in our scan}}%
  {\begin{tabular}{lccc}
  \toprule
  \bfseries Condition outcome & \bfseries 2018 & \bfseries 2019 & \bfseries 2020  \\
  \midrule
  Abnormal breast imaging & & & 1 \\
  Acute bronchitis & & & 1 \\
  Acute maxillary sinusitis & & & 1 \\
  Acute sinusitis & & & 1 \\
  Acute upper respiratory infection & & & 1 \\
  Arthralgia of ankle/foot & & & 1 \\
  Blood chemistry abnormal & & & 1 \\
  Chronic pain & & & 1 \\
  Cough & & & 1 \\
  Diverticulosis of large intestine & & & 1 \\
  Fatigue & & & 1 \\
  Fever & & & 1 \\
  Finding of frequency of urination & & & 1 \\
  Headache & & & 1 \\
  Hemorrhoids & & & 1 \\
  Hyperpigmentation of skin & & & 1 \\
  Impacted cerumen & & & 1 \\
  Impaired fasting glycemia & & & 1 \\
  Melanocytic nevus of trunk & & & 1 \\
  Muscle pain & & & 1 \\
  Nausea and vomiting & & & 1 \\
  Nuclear senile cataract & & & 1 \\
  Obesity & & & 1 \\
  Pain in left knee & & & 1 \\
  Pain in right foot & & & 1 \\
  Polyp of colon & & & 1 \\
  Pure hypercholesterolemia & & & 1 \\
  Skin changes due to chronic radiation & & & 1 \\
  Skin neoplasm & & & 1 \\
  Tear film insufficiency & & & 1 \\
  Upper respiratory tract infection & 1 & 1 & 1 \\
  Verruca vulgaris & & & 1 \\
  Vitamin D deficiency & & & 1 \\
  \midrule
  Total & 1 & 1 & 33 \\
  \bottomrule
  \end{tabular}}
\end{table*}

\begin{table*}[htbp]
\floatconts
  {tab:procedure_results}%
  {\caption{Non-stationary procedure outcomes in our scan}}%
  {\begin{tabular}{lccccc}
  \toprule
  \bfseries Procedure outcome & \bfseries 2016 & \bfseries 2017 & \bfseries 2018 & \bfseries 2019 & \bfseries 2020  \\
  \midrule
  Breast cancer screening & & & & & 1 \\
  Cervical screening & & & & & 1 \\
  Colonoscopy & & & & & 1 \\
  Esophagogastroduodenoscopy & & & & & 1 \\
  Eye services & & & & & 1 \\
  Gynecology & & & & & 1 \\
  Inpatient consultation & & & & 1 & \\
  Nursing care & 1 & & & & \\
  Oximetry & & & & & 1 \\
  Preventive medicine evaluation & & & & & 1 \\
  Spirometry & & & & & 1 \\
  Vaccination & & & & & 1 \\
  \midrule
  Total & 1 & 0 & 0 & 1 & 10 \\
  \bottomrule
  \end{tabular}}
\end{table*}

\begin{table*}[htbp]
\floatconts
  {tab:lab_results}%
  {\caption{Non-stationary lab outcomes in our scan. $^\textbf{*}$Lab was missed during cleaning, so there was some variation in reference ranges for different unit spellings.}}%
  {\begin{tabular}{lcccccc}
  \toprule
  \bfseries Lab outcome & \bfseries Abnormal range & \bfseries 2016 & \bfseries 2017 & \bfseries 2018 & \bfseries 2019 & \bfseries 2020  \\
  \midrule
  25-hydroxyvitamin D3 & $< 30$ ng/mL & 1 & & 1 & & 1 \\
  25-hydroxyvitamin D3 + D2 & $< 30$ ng/mL & 1 & & & & 1 \\
  Albumin/Globulin ratio & $> 2.2$ & 1 & & 1 & 1 & 1 \\
  Carbon dioxide & $< 20$ mmol/L & & & 1 & & 1 \\
  Cholesterol & $> 239$ mg/dL & & & & & 1 \\
  Cholesterol LDL & $> 129$ mg/dL & & 1 & & & \\
  Cholesterol LDL calculated & $> 129$ mg/dL & & & & & 1 \\
  Cholesterol LDL/HDL ratio & $> 3$ & & & & & 1 \\
  Cholesterol non-HDL & $> 129$ mg/dL & 1 & 1 & 1 & & 1 \\
  Creatinine & $< 0.74$ mg/dL (men), & & & & & 1 \\
  & $< 0.59$ mg/dL (women) & & & & & \\
  eGFR CKD-EPI & $< 60$ mL/min/1.73m$^2$ & 1 & 1 & 1 & 1 & 1 \\
  eGFR CKD-EPI Black & $< 60$ mL/min/1.73m$^2$ & & & 1 & 1 & \\
  eGFR CKD-EPI non-Black & $< 60$ mL/min/1.73m$^2$ & & & 1 & 1 & \\
  eGFR MDRD & $< 60$ mL/min/1.73m$^2$ & & 1 & & & \\
  eGFR MDRD non-Black & $< 60$ mL/min/1.73m$^2$ & & & 1 & & \\
  Erythrocyte distribution width & $< 11.8\%$ (men), & & 1 & 1 & & 1 \\
  & $< 12.2\%$ (women)$^\textbf{*}$ & & & & & \\
  Erythrocyte sedimentation rate & $> 22$ mm/hr (men) & & & & & 1 \\
  & $> 29$mm/hr (women) & & & & & \\
  Glucose & $> 125$ mg/dL & & & 1 & & \\
  Hematocrit & $> 48.6\%$ (men) & & & & & 1 \\
  & $> 44.9\%$ (women) & & & & & \\
  Iron saturation & $< 20\%$ & & & & & 1 \\
  Lymphocytes & $< 1.5$e$3$ cells/$\mu$L & & & & & 1 \\
  Lymphocytes/100 leukocytes & $< 20$ & & 1 & 1 & & \\
  Magnesium & $> 2.2$ mg/dL & & & & & 1 \\
  Microalbumin & $> 30$ mg/dL & 1 & & & & \\
  Neutrophils/100 leukocytes & $< 40$ or $< 1.8$e$3$/$\mu$L & & & & & 1 \\
  Platelet mean volume & $< 7$ fL & & & & 1 & 1 \\
  Rubella virus IgG Ab & $> .9$ & & & & & 1 \\
  Thyrotropin & $< 0.5$ $\mu$U/mL & 1 & & & & \\
  Thyrotropin & $> 5$ $\mu$U/mL & 1 & & & & \\
  Urobilinogen & $< 0.2$ mg/dL & 1 & & & 1 & \\
  \midrule
  Total & & 9 & 6 & 11 & 6 & 19 \\
  \bottomrule
  \end{tabular}}
\end{table*}

\end{document}